\documentclass{article}
\usepackage{ijcai26}

\usepackage{times}
\usepackage{soul}
\usepackage{url}
\usepackage[hidelinks]{hyperref}
\usepackage[utf8]{inputenc}
\usepackage[small]{caption}
\usepackage{graphicx}
\usepackage{amsmath}
\usepackage{amsthm}
\usepackage{booktabs}
\usepackage{algorithm}
\usepackage{algorithmic}
\usepackage[switch]{lineno}

\usepackage{xspace}
\usepackage{xcolor}

\usepackage{tabularx}
\usepackage{tikz}
\usetikzlibrary{arrows.meta, positioning}
\usepackage{amsmath}
\usepackage{amsthm}
\usepackage{amssymb}
\usepackage{stackengine}
\usepackage{paralist}
\usetikzlibrary{arrows}
\usepackage{subcaption}
\usepackage{xcolor}
\usepackage{xspace}
\usepackage{multirow}
\usepackage{pgfplots}

\newtheorem{theorem}{Theorem}

\newtheorem{example}{Example}

\newif\ifextendedversion
\extendedversiontrue   % toggle ON
\title{Transforming and Encoding FTS for SAT Solving: What Helps, What Hurts}

\author{
João Filipe$^1$
\and
Álvaro Torralba$^2$
\and
Gregor Behnke$^1$
\\
\affiliations
$^1$University of Amsterdam, Institute for Logic Language and Computation, The Netherlands\\
$^2$Aalborg University, Denmark\\
\emails
\{j.sa, g.behnke\}@uva.nl,
alto@cs.aau.dk
}

\newcommand{\transitionsystem}{\ensuremath{\Theta}\xspace}

\newcommand{\cT}{\ensuremath{\mathcal{T}}\xspace}

\newcommand{\SAS}{\ensuremath{\text{SAS}^+}\xspace}

\newcommand{\init}{\ensuremath{s^{I}}\xspace}

\newcommand{\states}{\ensuremath{{S}}\xspace}

\newcommand{\labels}{\ensuremath{{L}}\xspace}
\newcommand{\transitions}{\ensuremath{T}\xspace}
\newcommand{\goalstates}{\ensuremath{\states^G}\xspace}

\newcommand{\factor}[1]{\ensuremath{\transitionsystem_{#1}}\xspace}
\newcommand{\factorstates}[1]{\ensuremath{\states_{#1}}\xspace}
\newcommand{\factorgoalstates}[1]{\ensuremath{\goalstates_{#1}}\xspace}
\newcommand{\factorinit}[1]{\ensuremath{\init_{#1}}\xspace}
\newcommand{\factorgoal}[1]{\ensuremath{\goalstates_{#1}}\xspace}
\newcommand{\factortr}[1]{\ensuremath{T_{#1}}\xspace}

\newcommand{\productTransitions}{\ensuremath{\transitions_{\cT}}\xspace}
\newcommand{\productInitstates}{\ensuremath{\init_{\cT}}\xspace}
\newcommand{\productGoalstates}{\ensuremath{\goalstates_{\cT}}\xspace}
\newcommand{\productStates}{\ensuremath{\states_{\cT}}\xspace}
\newcommand{\productTheta}{\ensuremath{\Theta_{\cT}}\xspace}

\newcommand{\prodfts}[0]{{\ensuremath{\Theta_\otimes}}\xspace}

\newcommand{\tuple}[1]{\ensuremath{\langle #1\rangle}\xspace}

\newcommand{\interval}[3]{\ensuremath{#2\leq#1\leq#3}\xspace}

\newcommand{\varlabel}[2]{\ensuremath{#1^{#2}}\xspace}
\newcommand{\varlabelgroup}[3]{\ensuremath{#1^{#3}_{#2}}\xspace}
\newcommand{\varselfloop}[2]{\ensuremath{\mathit{SL}_{#1}^{#2}}\xspace}
\newcommand{\varsomeactual}[2]{\ensuremath{\mathit{NS}_{#1}^{#2}}\xspace}

\newcommand{\nonselflooplabels}{\ensuremath{\labels^{\rightarrow}_k}\xspace}

\newcommand{\timesteps}{\ensuremath{N}\xspace}

\makeatletter
\newcommand{\pushright}[1]{\ifmeasuring@#1\else\omit\hfill$\displaystyle{#1}$\fi\ignorespaces}
\newcommand{\pushleft}[1]{\ifmeasuring@#1\else\omit$\displaystyle{#1}$\hfill\fi\ignorespaces}
\makeatother

\begin{document}

\maketitle

\begin{abstract}
Factored tasks are a classical planning representation that extends SAS+ with limited forms of disjunctive preconditions, conditional effects, and angelic nondeterminism. This allows for a more compact representation of tasks than traditional formalisms such as STRIPS or SAS+, and supports a wide range of task transformations. However, existing planning approaches for factored tasks have been limited to heuristic search methods.

In this work, we investigate how to encode factored tasks in SAT. We propose several ways to encode the tasks, focusing on different strategies for translating the factored transition relation into propositional logic. We also analyze how to exploit parallelism at various levels in this setting and study the impact of common task transformations on the performance of SAT-based planners.
\end{abstract}

\section{Introduction}

Classical planning algorithms aim to exploit a model of the environment to find a
plan. Different representations can be used to model the environment, e.g.,
STRIPS~\cite{fikes-nilsson-aij1971} or \SAS~\cite{backstrom-nebel-compint1995}, offering a
trade-off between expressiveness and ease to exploit for planning algorithms. We consider
planning tasks in a factored transition system (FTS) representation, where states are
described in terms of finite-domain variables, and each variable (also called factor) is
described as a transition system. This provides a lot of flexibility for expressing the
possible behaviors (preconditions and effects) of an action within one variable, provided
independence of this behavior across variables.  The FTS representation was originally considered in the
context of merge-and-shrink abstractions~\cite{helmert-et-al-jacm2014}, but was later
adopted as a planning task representation due to having several advantages over previous
representations. Compared to \SAS, for example, we can easily express for each variable
disjunctive preconditions, conditional effects, and even angelic non-determinism (where we
can choose multiple outcomes for a given action). Expressing the same transition semantics
in \SAS requires either exponential-size compilation or introducing auxiliary state
variables and operators that increase plan
length~\cite{nebel-jair2000,buechner-et-al-icaps2024}. On the one hand, it allows to
naturally encode domains that are not easy to express in \SAS, such as finding algorithms
for matrix multiplication~\cite{speck-et-al-icaps2023}, or Rubik's
cube~\cite{korf-aaai1997,buechner-et-al-icaps2024}. On the other hand, it provides a very
flexible representation for task
transformation~\cite{torralba-sievers-ijcai2019,sievers-helmert-jair2021}.

Thus, it is desirable to have planning methods that support this representation
natively. So far, existing planning approaches for factored tasks have focused more on
optimal planning, using heuristic search methods with admissible
heuristics~\cite{torralba-sievers-ijcai2019,buechner-et-al-icaps2024}, or symbolic
search~\cite{torralba-et-al-aij2017,torralba-et-al-ipc2023}. For agile planning, only
heuristic search algorithms
exist~\cite{torralba-sievers-ijcai2019,torralba-et-al-ipc2023}. These are not
always suitable for the FTS representation.
For example, an FTS can
encode tasks with exponential branching factor in the size of the task, where
straightforward heuristic search approaches fail immediately. %% To address this, we
%% introduce SAT planning techniques for this representation.

SAT-based planning has been a prominent approach since the early work of
\citeauthor{kautz-selman-ecai1992}~\shortcite{kautz-selman-ecai1992}. The general idea is
to encode the existence of a plan of length $\timesteps$ as a propositional formula $\phi_N$,
which is satisfiable if and only if such a plan exists. This formula is then passed to a
SAT solver, either incrementally increasing $\timesteps$ for optimal planning, or following a
schedule for satisficing planning~\cite{rintanen-aij2012,rintanen-ipc2011}. Over the years,
several SAT-based planning systems have been developed, such as
SATPLAN04~\cite{kautz-et-al-ipc2006}, Madagascar~\cite{rintanen-aij2012}, action splitting~\cite{robinson-et-al-icaps2009}, abstract CNF encodings~\cite{domshlak-et-al-jair2009} PASAR~\cite{froleyks-et-al-socs2019}, AxSAT~\cite{behnke2025axsat},  and
SASE~\cite{huang-et-al-jair2012}, each contributing different encoding strategies or
solver techniques. However, these systems typically target classical representations and
are not directly applicable to FTS tasks. Adapting SAT planning to this setting enables
new encoding strategies that make better use of the structure inherent in factored models.

In this work, we consider several alternative ways to encode FTS planning tasks as a SAT
formula. While a naive encoding is straightforward, we propose several encodings for the
transition relation, as well as exploiting parallelism. Our experiments show that these
optimizations are fundamental to achieve a good performance, that SAT-based planning can
be competitive as well for solving factored tasks, and have good synergy with task
reformulation methods.

\section{Preliminaries}

A \emph{transition system} (TS) is a tuple
$\transitionsystem=\tuple{\states,\labels,\transitions,\init,\goalstates}$ where $\states$
is a finite set of \emph{states}, $\labels$ is a finite set of \emph{labels},
$\transitions \subseteq \states \times \labels \times \states$ is a set of
\emph{transitions}, $\init \in \states$ is the \emph{initial state}, and $\goalstates
\subseteq \states$ is the set of \emph{goal states}.
%
%We use $s\trans{\ell} t$ as a shorthand for $(s, \ell, t) \in \transitions$.
A \emph{path}
$s \xrightarrow{l_1} \dotso \xrightarrow{l_k} t$ from $s$ to $t$ is a sequence
of $k$ transitions $(s_{i-1}, l_i, s_{i}) \in \transitions$ for
$\interval{i}{1}{k}$, starting in $s=s_0$ and ending in $t=s_k$. A \emph{plan} for $s$ is
a path from $s$ to any $s' \in \goalstates$.

\subsection{Factored Tasks}

We consider planning tasks in a factored transition system representation, where states
are described in terms of finite-domain variables, and each variable (also called factor)
is described as a
TS~\cite{helmert-et-al-jacm2014,torralba-sievers-ijcai2019,sievers-helmert-jair2021,buechner-et-al-icaps2024}.
A factored task $\cT = \tuple{\factor{1},\dotsc,\factor{n}}$ is a tuple of TSs with a
common set of labels, \labels, such that $\factor{i} = \tuple{\factorstates{i}, \labels,
  \factortr{i}, \factorinit{i}, \factorgoal{i}}$ for $\interval{i}{1}{n}$. We assume that all $S_i$ are disjoint, i.e., $S_i \cap S_j = \emptyset$.
  We characterize
the size of task $\cT$ in terms of the number of factors $|\cT| = n$, the number of labels
$|L|$, and the maximum number of states in each factor $D = \max_{\factor{i} \in \cT}
|\factorstates{i}|$.

This is a compact representation of a state space, which is the synchronized product
$\prodfts = \bigotimes_{i=1}^n \Theta_i$ of all factors.
%% We use a tuple, instead of a set to refer to the individual factors by their index. 
%% To differentiate states in a factored task, and in each of the factors, \factor{i}, we
%% call \emph{facts} to the states in the individual factors $s_i \in \factorstates{i}$.  
A state in $\prodfts$ is a tuple of states from the individual factors $s =
\tuple{s_1,\dotsc,s_n}$, one for each factor, $s_i \in \factorstates{i}$.
The state space of a factored task is the synchronized product of all its factors. This is
another transition system $\productTheta = \factor{i} \otimes \dotsc \otimes \factor{n} =
\tuple{\productStates, \labels, \productTransitions, \productInitstates,
  \productGoalstates}$, where $\productStates = \factorstates{1} \times \dotsb \times
\factorstates{n}$ is the Cartesian product of each factor's states, $\productTransitions =
\{(\tuple{s_1,\dotsc,s_n},l, \tuple{t_1,\dotsc,t_n}) \mid (s_i, l, t_i) \in
\transitions_i \text{~for~} \interval{i}{1}{n}\}$, $\productInitstates =
\tuple{\init_1,\dotsc,\init_n}$, and $\productGoalstates = \factorgoalstates{1} \times
\dotsb \times \factorgoalstates{n}$.

%\alvaro{Add example}

%\begin{figure}
%\centering
%\begin{tikzpicture}[
%    >=Stealth,
%    node distance=2cm,
%    every node/.style={circle, draw, minimum size=0.8cm},
%    every loop/.style={looseness=12},
%    edge label/.style={draw=none, inner sep=1pt, font=\small}
%    ]
    
    % Nodes
%    \node (A) {A};
%    \node (B) [right=of A] {B};
%    \node (C) [right=of B] {C};

    % Edges
%    \draw[->] (A) to [out=15, in=165] node[midway, above, edge label] {$l_1, l_2$} (B);
%    \draw[->] (B) -- (C) node[midway, above, edge label] {$l_1$};
 %   \draw[->] (A) to [out=315, in=225] node[midway, above, edge label] {$l_1$} (C);
 %   \draw[->] (C) to [out=135, in=45] node[midway, below, edge label] {$l_1$} (A);

%    \draw[->] (B) to [out=195, in=345] node[midway, below, edge label] {$l_2$} (A);
    
%\end{tikzpicture}
%\caption{Example of a Factored Task with two factors\alvaro{Probably we want here a second factor, to be able to explain how they sync (and also so that parallelism can be illustrated}}
%\label{fig:ts_example}

%\end{figure}

\subsection{Task Transformation}

Apart from being able to easily express some planning tasks more naturally than other
formalisms, an advantage of using the factored representation is that a large number of
\emph{transformations} are available~\cite{sievers-helmert-jair2021}. Though many of these
transformations are tailored towards creating abstractions of the task, some of them can
also be used to simplify the task before solving it, if they have the property that a plan
for the original task can be reconstructed from any plan for the transformed
task~\cite{torralba-sievers-ijcai2019}. Some of the transformations available include:
\begin{itemize}
  \item Label reduction~\cite{sievers-et-al-aaai2014} reduces the set of labels while
    keeping the set of transitions of the product TS intact (up to label renaming).
  \item Weak-bisimulation shrinking~\cite{hoffmann-et-al-ecai2014} reduces the number of
    states in a factor. This may change the goal distance of states in \prodfts, but it
    preserves whether states were solvable or not so that a solution for the
    original task can be reconstructed.
  \item Pruning unreachable and dead-end states in the
    factors~\cite{sievers-helmert-jair2021}.
  \item Merging two or more factors into one~\cite{fan-et-al-socs2014,sievers-et-al-icaps2016,sievers-et-al-icaps2024}. This reduces the number of factors, at the cost
    of increasing their size exponentially in the number of combined factors.
\end{itemize}

These transformations can be combined, having
synergies~\cite{sievers-helmert-jair2021}. For example, pruning typically is not
applicable on the input task, but merging factors may lead to discovering that many of the
combined states are unreachable or dead-ends. Notably, the transformations make use of the
expressiveness of the representation. A factored task obtained from compiling a \SAS task
will never have non-deterministic labels or disjunctive preconditions. However, label
reduction or shrinking may introduce them.

\subsection{Boolean Satisfiability}

Boolean Satisfiability (SAT) is the task of determining whether a propositional formula in
conjunctive normal form (CNF) has a truth assignment of its variables that makes the
formula true.  A CNF formula consists of a conjunction ($\wedge$) of clauses, where each
clause is a disjunction ($\vee$) of literals.  A literal is a propositional variable, or
variable for short, $x$ or its negation $\neg x$.  SAT solvers, which are readily
available, determine whether a formula is satisfiable, and if so, provide a satisfying
truth assignment.  For ease of notation, we will not present the SAT formulae in this
paper in CNF, but all formulae can be translated easily to CNF.

\section{SAT Planning}

The general idea in SAT planning is to, given a planning task and a bound on the number of
timesteps $\timesteps$, construct a SAT formula that is satisfiable if and only if there exists a
plan with a makespan of $\timesteps$. Our encodings use the following variables, for each timestep
$t \in \{0, \dots, \timesteps\}$:
\begin{compactitem}
    \item $s_i^t$: indicates if we are at state $s_i \in \factor{k}$ at timestep $t$. %\alvaro{Clarify how to refer to states in factors}
    \item $\varlabel{l}{t}$: indicates if a label $l \in \labels$ is selected at timestep $t$.
\end{compactitem}

Then, we need to encode constraints that ensure that the values of these variables
correspond to an executable and goal-reaching path.
In each factor, we have to be exactly in one
state at every timestep. To enforce this, we assert that at least one state in each factor
has to be selected via a disjunction ($\vee$) across all states of that
factor. We also assert that at most one state in the factor is selected. Depending on the
number of variables, its encoding varies.  For up to 256 variables, a naive quadratic
encoding is used, whereas a binary counter encoding is applied
otherwise~\cite{sinz2005towards}.
\begin{align}
  & \mathit{atLeastOne}(\{s^t \mid s \in S_k\}) & \forall \factor{k} \in \cT  \label{constraint:atLeastOneState}\\
  & \mathit{atMostOne}(\{s^t \mid s \in S_k\}) & \forall \factor{k} \in \cT  \label{constraint:atMostOneState}
\end{align}

Additionally, the initial and goal states are enforced at the first and last timesteps
respectively.
\begin{align}
    & s_i^0 & \forall s_i \in \init \label{constraint:initState}\\
    \bigvee\nolimits_{s_i \in \factorgoalstates{k}} & s_i^\timesteps & \forall \factor{k} \in \cT \label{constraint:goalState}
\end{align}
Finally, we must ensure that consecutive layers are consistent with the transition
relation, i.e., for each $t$, $s^{t+1}$ can be reached from $s^t$ by applying the selected
labels $\varlabel{l}{t}$. Next, we focus on encoding the transition relation under the
restriction that a single label can be selected at each timestamp.
\begin{align}
    & atLeastOne(\{\varlabel{l}{t} \mid l \in \labels\}) \label{constraint:atLeastOneLabel}\\
    & atMostOne(\{\varlabel{l}{t} \mid l \in \labels\}) \label{constraint:atMostOneLabel}
\end{align}
 We will relax these constraints when we discuss parallelism.

%\alvaro{Say that for now, we assume that at each time-step only one label is selected. Mention that in the parallelism section we will relax this assumption.}
%Since the only difference between the three encodings lies in the constraints used which leverage better or worse the underlying structure of each transition system,

\section{Encoding of the Transition Relation}

We present several ways to encode FTS tasks into SAT, which leverage the structure of the FTS representation.
%% But before we look at how the problems are encoded into SAT, it is helpful to first
%% understand how we view them.  As mentioned earlier, an FTS task corresponds to the
%% synchronized product of several transition systems, all using the same set of labels.
In an FTS task, labels capture all the dependencies across factors.  So, it suffices to
encode each \factor{k} separately, restricting which pairs of states $s^t_i,s^{t+1}_j \in
S_k$ are consistent with the label \varlabel{l}{t} selected at time $t$.  %In this section,
We describe the encoding for a single factor \factor{k} and timestep $t$. The final SAT
formula is the conjunction of the encoding for every factor and time step, in addition
to constraints \ref{constraint:atLeastOneState} to \ref{constraint:atMostOneLabel}.

The transition relation $T_k$ can be seen as a 3D structure (with dimensions: source, target and label) indicating whether a transition is valid for a given triple of values.
In a SAT encoding, the constraints do not ``build solutions'' directly but rule out those assignments that do not correspond to possible transitions.
%
%Each label can have one or more transitions in each factor.
%The transition relation $T_k$, which is a structure with three dimensions: source, target, and label.
%As such, the transitions for a given label in a particular factor can be captured using an
%adjacency matrix.% as shown in Example~\ref{example1}.
%
%% Whenever an assignment makes one of those clauses false, it is immediately discarded.
%% In effect, the solver traverses only the valuations that have not been forbidden, and
%% any assignment that remains when no more clauses can be violated is guaranteed to meet
%% all of our problem's constraints.
The simplest encoding of $T_k$ is to explicitly list every impossible transition.
Since at every timestep a state and a label must be chosen (constraints \ref{constraint:atLeastOneState}, \ref{constraint:atMostOneState},  \ref{constraint:atLeastOneLabel} and \ref{constraint:atMostOneLabel}), the chosen transition must be one allowed in $T_k$.
We call these constraints forbidding transitions
%, which include only negative literals in their clauses,
\emph{primal constraints}. 
%
%by iterating over every cell with a zero in the
%three-dimensional structure and assert that the three
%corresponding values cannot be simultaneously selected:
\begin{align}
    %& \neg (l^t \wedge s_i^t \wedge s_j^{t+1}) \equiv \neg l^t \vee \neg s_i^t \vee \neg s_j^{t+1} & \forall (s_i, l, s_j) \not\in T_k \label{constraint:forbiddenCellTransition}
    & \neg (l^t \wedge s_i^t \wedge s_j^{t+1}) & \forall (s_i, l, s_j) \not\in T_k \label{constraint:forbiddenCellTransition}
\end{align}

Together constraints~(\ref{constraint:atLeastOneState}) through~(\ref{constraint:forbiddenCellTransition}) already constitute a complete encoding of the FTS task as a SAT formula -- which will be our baseline encoding. %\alvaro{I think we can comment this. The experiments are too far away for the reader to remember this name for the tables Seq (sequential) without any optimizations.}
This encoding can be quite inefficient in terms of the number of clauses, as constraint~\ref{constraint:forbiddenCellTransition} requires $O(|L|\cdot|\cT|\cdot{}D^2)$ clauses.
It is also fairly inflexible as it only asserts that transitions that are not present in the transition system cannot be used and does not leverage in any way the information of valid
transitions.
In the rest of the section we develop alternative ways of encoding
equivalent formulas with the same satisfying assignments.
%\alvaro{I think this is easier to explain because is basically: go cell for cell and forbid it. Also, there  is no choice in which of the dimensions are on the left hand or the right hand: there is some symmetry }

% \alvaro {I think something along the lines of the following proposition is needed to make it very clear that this simple way already works.}
% \begin{proposition}
% The encoding using Equations \alvaro{Add reference at the equations from the previous section and \ref{eq:naive_tr}} is "correct" \alvaro{How do we define that an encoding is correct? }
% \end{proposition}

\subsection{Constraint-per-line Encoding} \label{sec:dual}

%\alvaro{Issue with the previous encoding. Instead of saying $s \land l \implies A$, we say  $s \land l \implies \not B$, $s \land l \implies \not C$. This is a very common case, as any deterministic label has at most one target per source. . }

%Furthermore, any label representing an action with a well defined effect will only have a single target (independent

Constraint~\ref{constraint:forbiddenCellTransition} excludes transitions that do not exist in $T_k$.
Each such constraint can be written as an implication %More specifically, it negates the conjunction of three variables that would otherwise represent an invalid transition.
stating that if two of the values are selected, then the third one cannot be selected: e.g.\ if we focus on the target state: 
$l^t \wedge s_i^t \implies \neg s_j^{t+1}$. %With $targets(l, s_i)$ denoting the set of valid successor states when applying label $l$ in state $s_i$, this constraint can then be rewritten as:
We can now aggregate all such implications for label $l$ and source state $s_i$ into a single formula.
\begin{align}
    %& l^t \wedge s_i^t \implies \bigwedge_{s_j\notin targets(l,s_i)} \neg s_j^{t+1} & \forall l \in \labels, \forall s_i \in \factor{k} \label{constraint:negTransitionPerRow}\\
%%%%
    & l^t \wedge s_i^t \implies \bigwedge\nolimits_{(s_i,l,s_j) \not\in T_k} \neg s_j^{t+1} & \forall l \in \labels, \forall s_i \in \factor{k} \label{constraint:negTransitionPerRow}
\end{align}
When focusing on a single label $l$, the transitions for $l$ in $T_k$ describe an adjacency matrix between the source and target states.
Constraint~\ref{constraint:negTransitionPerRow} then encodes each row of $l$'s adjacency matrix by
forbidding targets which cannot be reached from the source corresponding to that row.
Transforming constraint~\ref{constraint:negTransitionPerRow} to CNF yields the same formula as constraint~\ref{constraint:forbiddenCellTransition}.

Since at any time at most one state is selected per factor, we can use a positive version of the implication instead of the primal constraint. % from constraint~\ref{constraint:forbiddenCellTransition}.
This \emph{dual} constraint expresses the  possible transitions of $T_k$.
Since at most one possible transition can be chosen, no impossible transition can be executed. % ensuring correct semantics.
%But considering entire rows gives us more flexibility, as we can alternatively express them by directly
%requiring that at least one valid target is selected: 
%. This leads to the following
%variant:
%
\begin{align}
    %& l^t \wedge s_i^t \implies \bigvee_{s_j\in targets(l,s_i)} s_j^{t+1} & \forall l \in \labels, \forall s_i \in \factor{k} \label{constraint:transitionPerRow}\\
%
    & l^t \wedge s_i^t \implies \bigvee\nolimits_{(s_i,l,s_j) \in T_k} s_j^{t+1} & \forall l \in \labels, \forall s_i \in \factor{k} \label{constraint:transitionPerRow}
\end{align}

For every row of the adjacency matrix of a label $l$ we can now either use the primal constraint constraint~\ref{constraint:forbiddenCellTransition} or the dual constraint~\ref{constraint:transitionPerRow}.
The key difference lies in the number and size of generated clauses.
Clauses of the primal constraint always have three literals, but require one clause per impossible target.
The dual constraint can always be expressed as one clause, but contains two plus the number of possible targets many literals.
I.e.\ using dual constraints for encoding yields fewer, but larger clauses.
%If the right-hand side of the implication contains more than one variable, the former generates one clause per invalid target, while the latter generates a single clause containing all valid targets, along with the label and source state variables.
Dual constraints typically have weaker propagation as the clauses are larger and unit-propagation triggers only when the value of all variables except one is known.
In order to avoid clauses with weaker propagation, we use the primal constraint~\ref{constraint:negTransitionPerRow} as a default.
Only when $|\{s_j \mid (s_i,l,s_j)\in T_k\}| = 1$ both primal and dual constraints have the same size (three).
Then we use the dual constraint to replace the primal constraints.
%Evaluation results using only primal constraints~\ref{constraint:negTransitionPerRow} are in the appendix.

\ifextendedversion
\begin{example}
Consider the  adjacency matrices for $l_1$ and $l_2$.
\\

\noindent
\begin{minipage}[t]{0.48\columnwidth}
  \centering
  \addtolength{\tabcolsep}{-0.4em}
    \begin{tabular}{c|c|c|c|}
         $l_1$&  A'&  B'& C'\\\hline
         A& 0 & 1 & 1 \\\hline
         B& 0 & 0 & 1 \\\hline
         C& 1 & 0 & 0 \\ \hline
    \end{tabular}
    %\captionof{table}{$l_1$ adjacency matrix}
    \captionof{table}{$l_1$ }
    \label{tab:l1_adjacency_matrix}
\end{minipage}%
\hfill
\begin{minipage}[t]{0.48\columnwidth}
    \centering
    \addtolength{\tabcolsep}{-0.4em}
    \begin{tabular}{c|c|c|c|}
         $l_2$&  A'&  B'& C'\\\hline
         A&  0&  1& 0\\\hline
         B&  1&  0& 0\\\hline
         C&  0&  0& 0\\ \hline
    \end{tabular}
    %\captionof{table}{$l_2$ adjacency matrix}
    \captionof{table}{$l_2$}
    \label{tab:l2_adjacency_matrix}
\end{minipage}\\

In the adjacency matrix for label $l_1$, it can be seen that there is no self-loop in $A$. In order to encode this, the following constraint would be added $\neg (l_1 \wedge A \wedge A')$, which is equivalent to $l_1 \wedge A \implies \neg A'$. Alternatively, one could also encode this constraint as $l_1 \wedge A \implies B' \vee C'$, however this constraint would have less propagation power. 
\label{example1}
\end{example}
\fi

\newcommand{\tikzcircle}[2][black,fill=black]{\tikz[baseline=-0.5ex]\draw[#1,radius=#2] (0,0) circle ;}%

\subsection{Projection Over 2 Dimensions}
Constraints~\ref{constraint:negTransitionPerRow} and \ref{constraint:transitionPerRow} each express an entire row of the adjacency matrix for a label $l$ using either a conjunction of negated values or a disjunction of positive ones on the right-hand-side of the implication.
%% The main difference between the two lies in the number and size of the generated
%% clauses when the right-hand side of the implication contains more than one variable.
%
%
But what happens if the right-hand side contains no literals?
For constraint~\ref{constraint:negTransitionPerRow}, we may transition from $s_i$ with label $l$ to any target -- rendering the constraint useless as it will always evaluate
to true.
On the other hand, constraint~\ref{constraint:transitionPerRow} becomes more interesting.
Here it is impossible to transition from $s_i$ with label $l$ to any target and the constraint~\ref{constraint:transitionPerRow} simplifies to $\neg s_i^t \lor \neg l^t$.
This clause allows us to represent an entire row of the adjacency matrix -- containing only zeros -- with a single binary clause.
This encoding is substantially more compact than constraint~\ref{constraint:negTransitionPerRow} using both fewer (only one) and smaller (size two) clauses.
In fact, this is a new type of primal constraint stating that any transition involving $s_i$ as the source and $l$ as the label is impossible.

Previously, we arbitrarily selected the target dimension for aggregation to obtain constraint~\ref{constraint:negTransitionPerRow}.
Symmetrically, we can consider the aggregation on sources and labels.
To create the dual constraints~\ref{constraint:transitionPerRow}, aggregating on targets is beneficial in our evaluation benchmark, as labels typically have few possible targets given a source.
If aggregation however leads to a new primal constraint with only two variables, any aggregation is beneficial as it allows to remove some of the original primal constraints~\ref{constraint:forbiddenCellTransition}.
We consider all three possible aggregations of the transition relation -- which we term the \emph{projection optimization} (``project away'' one dimension). % in the constraint).
%
%Formally we consider these aggregations of the transition relation:
\begin{align}
 T_k^{\mathtt s} [l][s_j] = \{ s_i \mid (s_i,l,s_j) \in T_k \} \\
 T_k^{\mathtt t} [s_i][l] = \{ s_j \mid (s_i,l,s_j) \in T_k \} \\
 T_k^{\mathtt l} [s_i][s_j] = \{ l \mid (s_i,l,s_j) \in T_k \} 
\end{align}
For any of the three (\textbf{s}ource, \textbf{t}arget, \textbf{l}abel), we then generate the following new primal constraints:
\begin{align}
 \neg x^{t(+1)}\!\lor\!\neg y^{t(+1)} \ \  \forall x,y : T_k^\sigma[x][y]\!=\!\emptyset\ \ \forall \sigma\!\in\! \{\mathtt{s,t,l}\}\label{constraint:projectedForbiddenTransition}
\end{align}
We can now omit all primal constraints~\ref{constraint:forbiddenCellTransition} that are already covered by these new binary primal constraints.
That is, we generate $\neg (l^t \wedge s_i^t \wedge s_j^{t+1})$ only if $T_k^s[l][s_j]$, $T_k^t[s_i][l]$, and $T_k^l[s_i][s_j]$ are non-empty. Otherwise a constraint~\ref{constraint:projectedForbiddenTransition} covering (i.e.\ implying) $\neg (l^t \wedge s_i^t \wedge s_j^{t+1})$ will be generated.

As for the primal constraint~\ref{constraint:forbiddenCellTransition}, we can aggregate the primal constraint~\ref{constraint:projectedForbiddenTransition} along one of the two remaining dimensions of the constraints to a constraint, e.g., $s_i \to \bigwedge_{T_k^t[s_i][l] = \emptyset} \neg l^t$.
We can again dualise these constraints to obtain constraints of the form $s_i \to \bigvee_{T_k^t[s_i][l] \neq \emptyset} l^t$ -- listing in this case the labels that are applicable in the state $s_i$.
These constraints have the same properties as the dual constraints involving all three dimensions (fewer but larger clauses).
In total, there are six types of such dual constraints, depending on which dimension is aggregated first and second.
%To avoid large clauses,
We again generate dualised clauses only if there is only one literal on the right-hand side. %, i.e., if one dimension uniquely determines the value of another.

In theory, it is possible that the right-hand side of these dualised constraints is empty, allowing to further project these 2D projections into a single dimension.
This would reveal information about individual states (that they are dead or unreachable), or labels (being inapplicable).
However any such state or label (except init or goal) are pruned automatically in preprocessing, making this further projection useless.

%that we could gain would be: labels that are not executable (which are automatically pruned when the task is created), states that have no incoming edges (which are automatically pruned unless it is the initial state) and states that have no outgoing edges (which are automatically pruned unless it is the goal state.)
%Let $v_1$ and $v_2$ be values from two different arbitrary dimensions, and $\mathcal{M}$ the matrix of projected values in those dimensions.
%The constraints generated from the 2D projection can be formally defined as follows:
% \begin{align}
%     & v_1 \implies \neg v_2 & \forall (v_1, v_2) \in \{ (x, y) \mid \mathcal{M}[x][y] = 0 \}
% \end{align}

\begin{example}
A simple way to visualise how these binary constraints are created is to consider once again the 3D structure $T_k$ and projecting its entries onto planes defined by pairs of dimensions. Each projected entry summarizes the corresponding line of $T_k$ parallel to the remaining axis. This entry is zero exactly when that line contains no valid transitions.
Consider the adjacency matrices for $l_1$ and $l_2$ and the projection (P) to the (source, target) plane. Here, we use different symbols to represent transitions forbidden by binary constraints from different planes: $\square$ for (label, source), $+$ for (label, target), and $\times$ for (source, target).
\\

\noindent
\begin{minipage}[t]{0.30\columnwidth}
  \centering
  \addtolength{\tabcolsep}{-0.4em}
    \begin{tabular}{c|c|c|c|}
         $l_1$&  A'&  B'& C'\\\hline
         A&  $\times$&  1& 1\\\hline
         B&  0&  $\times$& 1\\\hline
         C&  1&  $\times$& $\times$\\ \hline
    \end{tabular}
    \captionof{table}{$l_1$}
    \label{tab2:l1_adjacency_matrix}
\end{minipage}%
\hfill
\begin{minipage}[t]{0.30\columnwidth}
    \centering
    \addtolength{\tabcolsep}{-0.4em}
    \begin{tabular}{c|c|c|c|}
         $l_2$&  A'&  B'& C'\\\hline
         A&  $\times$&  1& \(+\)\\\hline
         B&  1&  $\times$& \(+\)\\\hline
         C&  $\square$ &  $\stackinset{c}{0pt}{c}{0pt}{\(\times\)}{\(\square\)}$ & $\stackinset{c}{0pt}{c}{0pt}{\(\square\)}{\stackinset{c}{0pt}{c}{0pt}{\(\times\)}{\(+\)}}$\\ \hline
    \end{tabular}
    \captionof{table}{$l_2$}
    \label{tab2:l2_adjacency_matrix}
\end{minipage}
\hfill
\begin{minipage}[t]{0.30\columnwidth}
    \centering
    \addtolength{\tabcolsep}{-0.4em}
    \begin{tabular}{c|c|c|c|}
         P&  A'&  B'& C'\\\hline
         A&  0&  1& 1\\\hline
         B&  1&  0& 1\\\hline
         C&  1&  0& 0\\ \hline
    \end{tabular}
    \captionof{table}{P}
    \label{tab2:projection}
\end{minipage}
\\

When looking at table~\ref{tab2:l2_adjacency_matrix}, we know that row "C" (a line
parallel to the "target" axis) and column "C'" (a line parallel to the "source" axis)
contain only 0's. As such, they can be encoded with constraints $l_2 \implies \neg C$
and $l_2 \implies \neg C'$.

When looking at table~\ref{tab2:projection}, we add the following constraints: $A \!\implies\! \neg A'$, $B \!\implies\! \neg B'$, $C \!\implies\! \neg B'$ and $C \!\implies\! \neg C'$.

After adding these constraints, only one 0 remains unblocked in the initial tables, which can be blocked with the following constraint: $l_1 \wedge B \implies \neg A'$. So with this method we were able to encode the information using 6 binary constraints and 1 ternary constraint as opposed to 12 (the total number of 0's in the initial tables) ternary constraints.

\end{example}

\subsection{Self-loops}

In each factor, many labels may have no effect at all.
A label $l$ is \emph{irrelevant} for factor \factor{k} iff it has a self-loop transition in every state in \factor{k} and no other transitions. All other labels are relevant.
Among the relevant labels there are further labels that only have self-loop transitions, but not necessarily in every state.
Encoding these labels as outlined so far, creates (with projection optimization) a clause $l^t \wedge s_i^t \implies s_i^{i+1}$ for each state $s_i$ in which $l$ has a self-loop.
These constraints are both repetitive and memory-inefficient, given that typically the majority of all labels is irrelevant for an individual factor.

To capture the semantics of self-loop transitions more compactly, we introduce an auxiliary variable $SL_k^t$:
\begin{compactitem}
    \item  $\varselfloop{k}{t}$: \factor{k} performed a self-loop transition at time $t$.
\end{compactitem}
We assert this semantic with the following constraint:
\begin{align}
%%    & \neg \varselfloop{k}{t} \wedge s_i^t \implies \neg s_i^{t+1} & \forall s_i \in S_k \label{contraint:diff_state}\\
%%    &  \varselfloop{k}{t} \wedge s_i^t \implies s_i^{t+1} & \forall s_i \in S_k \label{contraint:same_state}
&    s_i^t \implies (\varselfloop{k}{t} \iff s_i^{t+1}) & \forall s_i \in S_k \label{contraint:same_state}
\end{align}
Let's now consider $\nonselflooplabels = \{l \in L \mid \exists (s, l, t) \in T_k \text{~s.t.~} s\neq t\}$, i.e., the set of labels that include non-self-loop transitions in \factor{k}, and $\labels^\circ_k = L \setminus \nonselflooplabels$ the labels that are only self-loops.
If $\varselfloop{k}{t}$ is false, i.e., the transition was not a self-loop, some label from $\nonselflooplabels$ was selected at timestep $t$ (Eq.~\ref{constraint:selfloop_base}).
\begin{align}
    & \neg \varselfloop{k}{t} \implies \bigvee\nolimits_{l \in \nonselflooplabels} l^t \label{constraint:selfloop_base}
\end{align}
We cannot require that $\varselfloop{k}{t}$ implies a disjunction of $l \in \labels^\circ_k$ as a label in $\nonselflooplabels$ can have both self-loop and non-self-loop transitions.
%
%The semantics of the self-loop labels themselves.
If a self-loop label $l \in \labels^\circ_k$ is executed, constraint~\ref{constraint:atMostOneLabel} ensures that no label in $\nonselflooplabels$ is executed and thus $\varselfloop{k}{t}$ is true.
Constraint \ref{contraint:same_state} then enforces the self-loop.
For relevant only-self-loop labels (self-loops in some but not all states), we additionally need to ensure that they are executed in a state in which they have a self-loop via (both for time $t$ and $t+1$):
%
%Additionally, for the labels that contain only self-loops but are not irrelevant we still
%need to assert their preconditions and effects. Let $precs(l)$ be the set of states where
%$l$ can be executed.
%
\begin{align}
    & l^t \implies \bigvee\nolimits_{(s_i,l,s_i) \in T_k} s_i^{t(+1)} & \forall l \in \labels^\circ_k\label{constraint:selfloopPrecEff}
\end{align}

The transition semantics of labels $l \in \nonselflooplabels$ is enforced as before.
Since the semantics of labels $l \in \labels^\circ_k$ is encoded via constraint~\ref{constraint:selfloopPrecEff}, we omit all constraints otherwise restricting their execution.
We remove all primal containing $l^t$ and all dual constraints with $l^t$ on the left-hand side.
If an $l \in \labels^\circ_k$ appears on the right-hand side of a projected dual clause, the clause specifies \emph{possible} labels, e.g., in constraints of the type $s_i \implies \bigvee l$.
To replace these occurrences, we introduce a new SAT variable.
\begin{compactitem}
    \item \varsomeactual{k}{t}: some label in $\nonselflooplabels$ was executed.
\end{compactitem}
We define its semantics using the constraints:

\begin{align}
    &\varsomeactual{k}{t} \Rightarrow \bigvee\nolimits_{l \in \nonselflooplabels} l^t %\label{constraint:actual1}\\
    &\neg \varsomeactual{k}{t} &\Rightarrow \bigwedge\nolimits_{l \in \nonselflooplabels} \neg l^t \label{constraint:actual2}
\end{align}
Note that $\neg \varsomeactual{k}{t}$ implies $\varselfloop{k}{t}$, but not vice versa -- as labels may have both self-loop and non-self-loop transitions.

We replace all occurrences of $l \in \labels^\circ_k$ with $\neg \varsomeactual{k}{t}$ (and remove duplicates) stating that \emph{some} self-loop can be performed while constraint~\ref{constraint:selfloopPrecEff} ensures that the selected self-loop is possible.
For dual constraints where the right-hand-side are labels (i.e., source or target implies labels), there are almost always at least two possible labels as almost every state has a self-loop with an irrelevant label.
To still enable the utilization of dual clauses in these cases, we generate them only if there is at most one \emph{relevant} label on the right-hand-side.
I.e. we may generate clauses of the form $s_i^t \Rightarrow l^t \lor \neg \varsomeactual{k}{t}$.

%
% \gregor{This is technically true, but we did not implement it ...}
%Once again this can be done by using either a conjunction of negated forbidden values (as represented in constraint~\ref{constraint:selfloopPrecEff}) or a disjunction of the allowed values.
%
This encoding of self-loops introduces only two extra variables and $2|S_k| + |\labels^\circ_k| + |\nonselflooplabels| + 2$ constraints per transition system instead of $|\{\labels \backslash \nonselflooplabels\}| \times |S_k|$.
Additionally, it lets us eliminate constraint~\ref{constraint:atLeastOneLabel}.
If no label is selected, $\varselfloop{k}{t}$ must be true (constraint~\ref{constraint:selfloop_base}) and thus the states does not change (constraint~\ref{contraint:same_state}).
This ability is crucial when using more advanced methods for selecting horizon length (Algorithm C, see ~\cite{rintanen-ipc2011} and Sec.~\ref{sec:experiments}) for satisficing planning.

\ifextendedversion

\begin{example}

Consider the following adjacency matrices for labels $l_1$ to $l_4$ where $l_1$ and $l_3$ are irrelevant labels, $l_2$ is composed only of self-loops but is relevant and $l_4$ contains transitions that are not self-loops.
\\

    \noindent
    \begin{minipage}[t]{0.30\columnwidth}
      \centering
      \addtolength{\tabcolsep}{-0.4em}
        \begin{tabular}{c|c|c|c|}
             $l_1$,$l_3$&  A'&  B'& C'\\\hline
             A&  1&  0& 0\\\hline
             B&  0&  1& 0\\\hline
             C&  0&  0& 1\\ \hline
        \end{tabular}
        \captionof{table}{$l_1$ and $l_3$}
        \label{tab3:l1_adjacency_matrix}
    \end{minipage}%
    \hfill
    \begin{minipage}[t]{0.30\columnwidth}
        \centering
        \addtolength{\tabcolsep}{-0.4em}
        \begin{tabular}{c|c|c|c|}
             $l_2$&  A'&  B'& C'\\\hline
             A&  1&  0& 0\\\hline
             B&  0&  0& 0\\\hline
             C&  0&  0& 1\\ \hline
        \end{tabular}
        \captionof{table}{$l_2$}
        \label{tab3:l2_adjacency_matrix}
    \end{minipage}
    \hfill
    \begin{minipage}[t]{0.30\columnwidth}
        \centering
        \addtolength{\tabcolsep}{-0.4em}
        \begin{tabular}{c|c|c|c|}
             $l_4$&  A'&  B'& C'\\\hline
             A&  0&  1& 0\\\hline
             B&  0&  0& 1\\\hline
             C&  0&  0& 0\\ \hline
        \end{tabular}
        \captionof{table}{$l_4$}
        \label{tab3:l4_adjacency_matrix}
    \end{minipage}
    \\

Without the self-loop optimisation, we create constraints of the type $l_1 \wedge A \implies A'$ for every state, for labels $l_1$ and $l_3$, and for label $l_2$ we create the same type of constraint for every state where there is a self-loop plus constraints like $l_2 \implies \neg B$ (if using the projection optimisation) for every state where there is not a self-loop. Therefore, the number of clauses to encode self-loop transitions scales with $O(|S_k||\labels^\circ|)$.

With the self-loop optimisation we start by encoding the constraint $A \implies (\varselfloop{}{} \iff A')$ for every state, for a total of $2|S_k|$ clauses. Then, we assert the preconditions of the relevant labels comprised only of self-loops, which in this case would be the constraint $l_2 \implies A' \vee C'$ adding an extra $|\labels^\circ|$ clauses. Then, we add the constraint enforcing that if $\varselfloop{}{}$ is false, then the transition executed was not a self-loop accounting for one constraint $\neg \varselfloop{}{} \implies l_4$. Lastly, we include the constraints using the $\varsomeactual{}{}$ variable. This corresponds to $\varsomeactual{}{} \implies l_4$ for an additional clause and to $\neg \varsomeactual{}{} \implies \neg l_4$ for $|\nonselflooplabels|$ additional clauses. So, we get $O(|S_k| + |\labels|)$, which is often smaller than the encoding without the optimisation. 

\end{example}

\fi

\subsection{Label Groups}

Similar to the optimization for self-loops, we can reduce redundancy in the encoding by
exploiting the structure of label groups. Label groups are an optimisation for
representing more compactly the transitions of each factor~\cite{sievers-socs2018}. This
is based on the observation that oftentimes there are many labels that share the same set
of transitions within a factor. For example, all irrelevant labels have a self-loop
in every state. Therefore, a label group $LG \subseteq L$ is a set of labels such that, for
any $l,l' \in LG$, $(s_i, l, s_j) \in \factortr{k} \iff (s_i, l', s_j) \in \factortr{k}$.  The same
idea can be translated into the SAT encoding. In each factor, all labels within a given
label group share the exact same set of transitions. As a result, encoding identical
constraints for each label individually is inefficient. To avoid this duplication, we
introduce one variable $\varlabelgroup{LG}{}{t}$ per label group for each factor.
%\begin{itemize}
%    \item  : indicates if a label group is selected at timestep $t$.
%\end{itemize}
%
Then the transition constraints can be encoded using label groups instead of labels. To
maintain correctness across factors, we add an additional constraint that ensures a label can only be
selected if its corresponding label group is also selected. Let $\mathcal{LG}$ be the set
of all label groups in \factor{k}. We then add for all $LG \in \mathcal{LG}$ the constraint
$\varlabelgroup{LG}{}{t} \Leftrightarrow \bigvee_{l \in LG} l^t$.

\begin{example}
Consider the adjacency matrices of labels $l_1$ and $l_2$ which execute exactly the same transitions under a certain transition system.
\\

\noindent
\begin{minipage}[t]{0.48\columnwidth}
  \centering
  \addtolength{\tabcolsep}{-0.4em}
    \begin{tabular}{c|c|c|c|}
         $l_1$&  A'&  B'& C'\\\hline
         A&  0&  1& 1\\\hline
         B&  0&  0& 1\\\hline
         C&  1&  0& 0\\ \hline
    \end{tabular}
    \captionof{table}{$l_1$}
    \label{tab4:l1_adjacency_matrix}
\end{minipage}%
\hfill
\begin{minipage}[t]{0.48\columnwidth}
    \centering
    \addtolength{\tabcolsep}{-0.4em}
    \begin{tabular}{c|c|c|c|}
         $l_2$&  A'&  B'& C'\\\hline
         A&  0&  1& 1\\\hline
         B&  0&  0& 1\\\hline
         C&  1&  0& 0\\ \hline
    \end{tabular}
    \captionof{table}{$l_2$}
    \label{tab4:l2_adjacency_matrix}
\end{minipage}
\\

Without the label group optimisation, the encoding of the constraints for the transitions of these two labels would be nearly identical, with the only difference in the constraints being the variable of the label, e.g., $l_1 \wedge A \implies \neg A'$ and $l_2 \wedge A \implies \neg A'$.

With the use of label groups, we can reduce all nearly identical constraints by introducing a new variable ($LG$) and a new constraint for each label group. With this new variable we can just encode the transitions using it, e.g., $LG \wedge A \implies \neg A$ and assert that if a label group is being used, then at least one of its labels was selected: $LG \Leftrightarrow l_1 \vee l_2$.
\end{example}
\section{Parallelism}

So far, our encoding has been restricted to generating sequential plans due to
constraint~\ref{constraint:atMostOneLabel}, restricting us to at most one label per timestep.
While this ensures correctness, it is highly limiting in
practice, as it significantly increases the number of steps required to reach a
solution~(see~e.g.~\cite{kautz-selman-aaai1996,rintanen-et-al-aij2006}).
We show how to extend our encoding to support parallelism,
i.e., selecting multiple labels simultaneously.
Removing
constraint~\ref{constraint:atMostOneLabel} without additional restrictions could result in
incorrect plans. There would be nothing to prevent labels from, e.g., executing
transitions $(s_i, l_1, s_j)$ and $(s_i, l_2, s_j)$ concurrently in one factor, which may
be invalid, e.g., %unless additional conditions are satisfied -- notably, 
if $s_j$ has no outgoing transitions.

We aim for a form of parallelism akin to the $\forall$-step parallelism for STRIPS and
\SAS tasks~\cite{kautz-selman-aaai1996,rintanen-et-al-aij2006}.  In $\forall$-step
parallelism, a set of actions can be applied in parallel if \emph{every} linearisation is
executable and leads to the same state. That is, if $a_1$ and $a_2$ are $\forall$-step
parallel, then for any state $s$ in which both are applicable, both $\langle a_1, a_2
\rangle$ and $\langle a_2, a_1 \rangle$ are valid and reach the same state.

%% , i.e., when considering a \SAS task and its compilation into FTS if we have set of
%% actions that are parallel in the \SAS task, then the corresponding labels in the FTS
%% task should be executable in parallel as well.

We consider two different definitions, one of which is simpler to encode, and one that
allows for more parallelism.

\subsection{Self-loop Parallelism}
%As a first step, we propose a simple and more restrictive version of $\forall$-step
%parallelism.
Consider transitioning in some factor from state $s_i$ to $s_j$ with a label $l$.
Any label $l'$ that is a self-loop in both $s_i$ and $s_j$ can be safely executed together with $l$ in any order.
%Leveraging our earlier optimization for self-loops for an efficient encoding,
We allow this parallelism only if $l'$ comprises \emph{only} of self-loop transitions.
That is, self-loop parallelism allows to execute at most \emph{one} actual transition in each factor together with a set of self-loop only transitions.
To enforce this, we replace constraint~\ref{constraint:atMostOneLabel} with a new set of
constraints~\ref{constraint:atMostOneNonSelfloopLabel}, one per factor. These constraints
ensure that at most one label from $\nonselflooplabels$ is selected at each timestep.
\begin{align}
    & atMostOne(\{l^t \mid l \in \nonselflooplabels\}) \label{constraint:atMostOneNonSelfloopLabel}
\end{align}
However in the base encoding of the transition relation $T_k$, selecting two labels to be true requires both to be executed.
This makes it impossible to execute a self-loop label $l^\circ$ together with a label $\vec{l}$ performing an actual transition: $l^\circ$ requires source and target state to be the same, while $\vec l$ requires them to be different.
I.e.\ parallelism between two labels is only possible if both of them are self-loops in all factors.

Here we leverage the self-loop optimization.
It enforces for each executed self-loop label $l^\circ$ only that both the source and the target state have a self-loop with $l^\circ$ (constraint~\ref{constraint:selfloopPrecEff}).
Source and target state must only be identical if no label from $\nonselflooplabels$ is executed.
Executing self-loops in parallel with one actual transition is still possible if $\varselfloop{k}{t}$ is false.

Lastly, dualised constraints where the right-hand-side contains labels become critical, as they relied on at most one label being executed at each time.
This will be a projected constraint of the from $s_i^{t(+1)} \Rightarrow \bigvee l^t$ with $s_i$ either being a source or target state.
Suppose a label $l \in \nonselflooplabels$ would be executed that is impossible, i.e., not mentioned on the right-hand-side of the dualised constraint.
Then no other $l' \in \nonselflooplabels$ can be selected for execution due to constraint~\ref{constraint:atMostOneNonSelfloopLabel}.
To satisfy the dualised constraint, the right-hand-side must contain $\neg \varsomeactual{k}{t}$ and $\varsomeactual{k}{t}$ must be false.
But this forces via constraint~\ref{constraint:actual2} all non-self-loop-only labels to also be false, i.e.\ $l$ must be false. %\alvaro{what?, rephrase this last sentence}
%Note that we cannot use $\varselfloop{k}{t}$ 

%For self-loop parallelism, the self-loop optimization enables us to 
%Selecting a self-loop label for execution 

%On its own, constraint \ref{constraint:atMostOneNonSelfloopLabel} forbid all pair of labels that are not $\forall$-parallel.
%Constraint~\ref{constraint:selfloopPrecEff} further ensures that all self-loop only labels executed in parallel have a self-loop in both the source and the target state.
%\gregor{Todo: continue re-writing here}

%Rather, this also relies on the encoding of the transition
%relation. When one label is just a self loop on a state and another label is just a self loop on a different state,  is the one preventing them from being executed in parallel. As a result, this form of parallelism is very weak unless the self-loop optimization is
%enabled.

\subsection{Chain Parallelism}

Self-loop parallelism is by construction a subset of
$\forall$-step parallelism: every set of actions $A$ executable under
self-loop parallelism is executable in parallel under $\forall$-step semantics.
%It is however a strict subset.

As previously mentioned, labels that contain transitions of the form $(s_i, l_1, s_j)$
and $(s_i, l_2, s_j)$, should not be applied simultaneously. However, if both labels also
allow for self-loop transitions at the target $s_j$ (i.e.\ $(s_j, l_1, s_j)$ and $(s_j, l_2, s_j)$), then any ordering of the $l_1$ and $l_2$ is safe and leads to the same resulting state $s_j$ -- with the idea that the label that is executed first transitions to $s_j$ and the other loops there.

This motivates a further notion of parallelism. 
We say that a set of labels $\mathcal L$ is
executable under \emph{chains parallelism} when transitioning from state $s_i$ to state $s_j$, if
when we consider the non-self-loop labels $\mathcal L \cap \nonselflooplabels$, either $|\mathcal L \cap \nonselflooplabels| = 1$ or for all $l \in \mathcal L \cap \nonselflooplabels$ the
transition $(s_j,l,s_j)$ exists.

%When encoding this form of parallelism,
We cannot simply allow or forbid a specific set of labels to be executed in parallel, since the validity of their parallel execution is dependent on $s_j$.
Let $\mathcal{L}_j$ be the set of labels which have a transition to state $s_j$ from another state.
We need to encode that if some $l \in \mathcal{L}_j$ is executed all other $l' \in \mathcal{L}_j \setminus \{l\}$ have a self-loop at $s_j$.
And thus if at least two are executed, all have a self-loop at $s_j$.
This exclusion is exactly the restriction that \citeauthor{rintanen-et-al-aij2006}'s~[\citeyear{rintanen-et-al-aij2006}] \emph{chains} constraints describe -- with $\mathcal{L}_j$ being the erasers and $\{l \mid l \in \mathcal{L}_j\text{ and }(s_j,l,s_j) \not\in T_k\}$ being the requirers.
%
%, and
%$\overrightarrow{\mathcal{L}_j}$ be the subset of labels without a self-loop at that state.
For every factor $\factor{k}$ and target state $s_j$ we introduce $2|\mathcal{L}_j|$ ``helper'' variables ($h_{k,j}^i$ and ${{h'}^{i}}_{k,j}$) along with constraints~\ref{constraint:topHelper1}
to~\ref{constraint:bottomHelper3}. % -- which express the chains construction of .
\begin{align}
    & (l_i \wedge s_j) \implies h^i_{k,j} \land h'^{i-1}_{k,j} &\forall l_i \in \mathcal{L}_j\label{constraint:topHelper1}\\
    & (h^i_{k,j} \implies h^{i+1}_{k,j})  &\!\!\!\! \forall i\!:\! 1 \!\leq\! i\! \leq\!  |\mathcal{L}_j|\!-\!2 \label{constraint:topHelper2}\\
    & (h'^{i}_{k,j} \implies h'^{i-1}_{k,j}) & \forall i\!:\! 2 \!\leq\! i\! \leq\!  |\mathcal{L}_j|\!-\!1 \label{constraint:bottomHelper2}\\
    %&  \implies \neg l_{i} & \forall l_i \in \overrightarrow{\mathcal{L}_j} \label{constraint:topHelper3}\\
    %& l_i \wedge s_j \implies   & \forall l_i \in \mathcal{L}_j \label{constraint:bottomHelper1}\\
    %&  & \forall i \in \{1, \dots, |\mathcal{L}_j|-2\} \label{constraint:bottomHelper2}\\
    & (h^{i-1}_{k,j} \lor h'^i_{k,j})\!\! \implies \!\!\neg l_i \ \ \  \forall l_i \in \mathcal{L}_j\!\!\!\!\! & \text{with }(s_j,l_i,s_j) \notin \transitions_k\label{constraint:bottomHelper3}
\end{align}

\begin{figure}[htbp]
    \centering
    \begin{tikzpicture}[every node/.style={circle, draw, minimum size=0.6cm,
    font=\small}, thick]
    
      % Middle row (at y=0)
      \node (M1) at (0,0) {$l_1$};
      \node (M2) at (3,0) {$l_2$};
      \node (M3) at (6,0) {$l_3$};
    
      % Top row (at y=2), centered between M1-M2 and M2-M3
      \node (T1) at (1.5,0.6) {$h_1$};
      \node (T2) at (4.5,0.6) {$h_2$};
    
      % Bottom row (at y=-2), also centered
      \node (B1) at (1.5,-0.6) {$h_1'$};
      \node (B2) at (4.5,-0.6) {$h_2'$};
    
      % Solid edges
      \draw[->] (M1) -- (T1);
      \draw[->] (M2) -- (T2);
      \draw[->] (T1) -- (T2);
      \draw[->] (B2) -- (B1);
      \draw[->] (M3) -- (B2);
      \draw[->] (M2) -- (B1);

      % Dashed edges
      \draw[->, dashed] (B2) -- (M2);
      \draw[->, dashed] (T1) -- (M2);
      \draw[->, dashed] (T2) -- (M3);
      \draw[->, dashed] (B1) -- (M1);
    
    \end{tikzpicture}
    \caption{Visual representation of constraints~19
    %\ref{constraint:topHelper1}
    to~21}
    %\ref{constraint:bottomHelper3}}
    \label{fig:chains}
\end{figure}
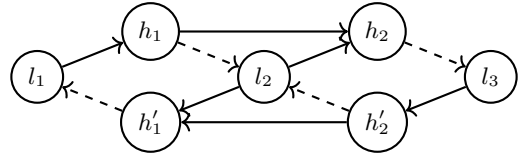
The core idea is the following: if a label $l$ is selected and leads to state $s_j$, then every
other label that can transition to $s_j$ but lacks a self-loop at $s_j$ must be
disabled.
Figure~\ref{fig:chains} gives a visual idea behind the encoding.
A normal edge
represents an implication from the source to the target and a dashed edge represents an
implication from the source to the negated target. Note that the dashed edges should only
be present if the target does not contain a self-loop.

In contrast to self-loop parallelism, chains parallelism allows for executing two non-self-loop labels at the same time step.
Thus moving from primal to dual constraints (see Sec.~\ref{sec:dual}) is not sound any more, i.e., the dual constraints do not forbid all impossible transitions, (see counterexample in the appendix
%~\cite{arxivappendix}
)
but only force that one of the selected labels is a possible transition.
We therefore cannot omit primal constraints whenever we generate a dual constraint.
As dual constraints provide useful information for unit-propagation, we nevertheless keep them.
I.e.\ for chains parallelism we include both primal and dual constraints and only omit primal constraints~\ref{constraint:negTransitionPerRow}, if there is a primal constraint~\ref{constraint:projectedForbiddenTransition} covering it.

\ifextendedversion
\begin{example}

Consider labels $l_1$ and $l_2$ where they both have the transitions $A->B'$, $A->C'$ and $C->C'$, but then only $l_2$ has the transition $B->B'$.
\\

\noindent
\begin{minipage}[t]{0.48\columnwidth}
  \centering
  \addtolength{\tabcolsep}{-0.4em}
    \begin{tabular}{c|c|c|c|}
         $l_1$&  A'&  B'& C'\\\hline
         A&  0&  1& 1\\\hline
         B&  0&  0& 0\\\hline
         C&  0&  0& 1\\ \hline
    \end{tabular}
    \captionof{table}{$l_1$}
    \label{tab5:l1_adjacency_matrix}
\end{minipage}%
\hfill
\begin{minipage}[t]{0.48\columnwidth}
    \centering
    \addtolength{\tabcolsep}{-0.4em}
    \begin{tabular}{c|c|c|c|}
         $l_2$&  A'&  B'& C'\\\hline
         A&  0&  1& 1\\\hline
         B&  0&  1& 0\\\hline
         C&  0&  0& 1\\ \hline
    \end{tabular}
    \captionof{table}{$l_2$}
    \label{tab5:l2_adjacency_matrix}
\end{minipage}
\\

From our encoding, these two labels can only be executed together under $\forall$-step parallelism if they are moving to state $C'$. Let's see why: For starters, for both labels to even have the possibility of being executed at the same time, their preconditions have to be satisfied, which means we are either starting at $A$ or at $C$ since $l_1$ is not executable in $B$. If we start at $A$ and try to move to $B'$, then the chains for this state would generate the constraints: $l_1 \wedge B' \implies h_1$, $l_2 \wedge B' \implies h'_1$ and $h'_1 \implies \neg l_1$, thus making it impossible to select both $l_1$ and $l_2$ while moving to $B'$, which is the intended behaviour since not all linearisations are possible, e.g., it is not possible to execute the labels in the order $l_2, l_1$ since $l_1$ does not have a self-loop at $B'$. 
%Note that constraint~\ref{constraint:topHelper3} is not generated for $h_1 \implies \neg l_2$ because $l_2$ has a self-loop at state $B'$. 
Then if we were to move to state $C'$, we would only require the constraints $l_1 \wedge C' \implies h_1$ and $l_2 \wedge C' \implies h'_1$ from the chains corresponding to this state, as the constraint $h'_1 \implies \neg l_1$ would no longer be generated since $l_1$ has a self-loop at $C'$, thus making it possible to execute both labels in parallel when moving from $A$ to $C'$. If we were to try to move from $C$ to $C'$ exactly the same constraints would be generated since once again both labels have a self-loop at the target.

This means that in total there are only 4 ways in which these labels can be executed in parallel: $(A,l_1,C'), (C,l_2,C')$, $(A,l_2,C'), (C,l_1,C')$, $(C,l_1,C'), (C,l_2,C')$ and $(C,l_2,C'), (C,l_1,C')$.

\end{example}
\fi

\begin{theorem}
If a set of labels $L$ is executable in parallel under chains parallelism then they are $\forall$-step parallel.
\end{theorem}
\begin{proof}
    Suppose $L$ is executable in parallel under chains parallelism.  Let $\vec L$ be an
    arbitrary ordering of $L$ and we have to prove that $\vec L$ is executable and all
    linearisiations of $L$ yield the same state.  $\vec L$ is executable if it is
    executable for every individual factor. Consider an arbitrary one. 
    Any label $l \in L \cap \labels^\circ$ does not change the state and can be executed at any position and thus be ignored.
    Observe that the remaining actions perform overall a single state transition from state $s_i$ to state $s_j$.
    In the execution of $\vec L$, the first label $l \in \vec L \setminus \labels^\circ$ can (as enforced by the primal and dual constraints) and will perform the transition. 
    %\gregor{Stronger: it is acutally the first action $\vec A \setminus %\labels^\circ$ - i.e. the first non-self-loop will do}
    Given constraints~\ref{constraint:topHelper1} to~\ref{constraint:bottomHelper3}, either $|L \setminus \labels^\circ| = 1$ or all other $l' \in \vec L \setminus \labels^\circ$ have a self-loop at $s_j$ allowing them to be executed and that execution leading to $s_j$.
\end{proof}

\subsection{FTS Parallelism vs $\forall$-step Parallelism}
Next, we study the relation of self-loop and chains parallelism with classical $\forall$-step parallelism on FTS and \SAS tasks.

%We can now compare the two notions of parallelism to classical $\forall$-step
%paralelism. To that end,
Consider how \SAS tasks are compiled into FTS tasks.
In \SAS tasks actions are specified via preconditions and effects, which are (partial) variable
assignments. When translated into FTS, each variable corresponds to a factor (the atomic
projection over that variable) and each action corresponds to a label. Thus, the
transitions for label $l$ in $\factor{k}$ belong to one of four cases, depending of
whether the value of the variable $k$ is defined for the preconditions and effects:
\begin{compactitem}
\item $v_k \not\in \mathit{pre}(l)$, $v_k \not\in \mathit{eff}(l)$: there is a self-loop on all states in \factor{k}
%\item $v_k \not\in \mathit{pre}(l) \cup \mathit{eff}(l)$: there is a self-loop on all states in \factor{k}
\item $v_k \in \mathit{pre}(l)$, $v_k \not\in \mathit{eff}(l)$: there is a single self-loop transition $(\mathit{pre}(l)[k], l, \mathit{pre}(l)[k])$
\item $v_k \not\in \mathit{pre}(l)$, $v_k \in \mathit{eff}(l)$: there are transitions $(s, l, \mathit{eff}(l)[k])$ from each $s \in \factor{k}$
\item $v_k \in \mathit{pre}(l)$, $v_k \in \mathit{eff}(l)$: there is a single transition $(\mathit{pre}(l)[k], l, \mathit{eff}(l)[k])$
\end{compactitem}
Self-loop parallelism is provably weaker than regular $\forall$-step parallelism.
Consider an \SAS{} problem with two actions $a_1$ and $a_2$ so that $a_1$ has the effects
$v_1=1$ and $v_2=2$, while $a_2$ has $v_2=2$ and $v_3=3$.  None of the actions have
preconditions.  $a_1$ and $a_2$ can clearly be executed in any order and all such orders
lead to the same state.  However w.r.t.\ the variable $v_2$ none of the two actions is a
self-loop and thus only one of them can be executed in parallel in our encoding.

However, chains parallelism is a generalisation of the $\forall$-step parallelism on \SAS tasks.

\begin{theorem}
Consider a FTS task whose factors are atomic projections of a \SAS task. Then, if $L$ is
executable in parallel under the classical $\forall$-step parallelism, it is also
executable under chains parallelism.
\end{theorem}
\begin{proof}
    Let $L$ be a set of labels executable under $\forall$-step parallelism.
    We can again ignore all irrelevant labels and focus on a single factor.
    We consider transitioning from state $s_i$ to state $s_j$ (if $s_i=s_j$ all executed transitions for that factor are self-loops, trivialising the conditions).
    Let $l \in L$ be an executed relevant label.
    There are linearisations of $L$ in which $l$ is the first label and one in which $l$ is the last label.
    From the latter case, we know that there must be a transition for $l$ ending at $s_j$ (either from $s_i$ or a self-loop). From the former, we know that there is a transition for $l$ that starts at $s_i$ (either a self-loop at $s_i$ or transitioning to $s_j$).
    If $l$ is a self-loop in both $s_i$ and $s_j$, then $l$ is a self-loop in all states if this factor is an atomic projection from a \SAS problem.
    Otherwise $l$ must have a transition $(s_i,l,s_j)$.
    Suppose it does not have the transition $(s_j,l,s_j)$.
    If this is the only non-self-loop label for the factor, no problem arises. Otherwise, there must be another label $l'$ for which the same reasoning holds.
    Then both $ll'$ and $l'l$ must be executable and reach the state $s_j$.
    Note that it is impossible for $l$ to have the transition $(s_i,l,s_i)$ -- as in the \SAS sense it has an effect moving it to $s_j$.
    The same holds for $l'$, thus if executed, $l'$ forces a transition to $s_j$ -- as it cannot loop at $s_i$.
    Since $l'l$ must be executable due to $\forall$-step parallelism, $l$ must be able to loop at $s_j$, thus forcing the transition $(s_j,l,s_j)$, which satisfies the chains-parallelism condition.    
\end{proof}

\section{Experiments} \label{sec:experiments}

%\paragraph{Introduction}
We implemented our encodings on FTSPlan~\cite{torralba-et-al-ipc2023}, and ran experiments on all International Planning Competition 
%\footnote{https://github.com/aibasel/downward-benchmarks}
domains, containing in total 2026 instances. We also ran experiments on the FTS benchmark which contains a total of 431 instances.
The experiments were run using Lab~\cite{seipp-et-al-zenodo2017} on an AMD EPYC 9654 with a memory limit of 1.75GB and a time limit of 1800 seconds. Our encodings use the Kissat SAT solver~\cite{BiereFallerFazekasFleuryFroleyksPollitt-SAT-Competition-2024-solvers}.
Madagascar comes with its own solver specifically tailored to its encoding and inextricably linked to it, as such it was ran with it.

We conducted two sets of experiments using a modified version of Fast Downward~\cite{helmert-jair2006} which supports the FTS representation (except for Madagascar) to evaluate our strategies. The first set (Table~\ref{tab:res1}) uses the transformations label reduction (LR) and shrinking (S) and presents the results of the baseline encoding (which encodes transitions only using constraint~\ref{constraint:forbiddenCellTransition}) as well as the results of encodings using projections (P), self-loop (SL) and label grouping (LG) optimisations with different parallelism strategies (Seq, for no parallelism; S-L for Self-loop parallelism; Chains for Chains parallelism) as well as different strategies for formula generation (One-by-One for generating a new formula with a length increased by 1 regarding the previous formula when that one is proved unsatisfiable; AlgC, similar to what is proposed by~\cite{rintanen-ipc2014}). The second set (Table~\ref{tab:res2}) explores the effect of applying various task transformations on top of the best-performing configuration, allowing us to better understand their impact. For reference, we also include the performance of the state-of-the-art SAT-based planner Madagascar (MpC), which was ran with configurations for $\forall$ and $\exists$-step parallelism combined with AlgC formula generations, as well as an adapted version of the FF heuristics~\cite{hoffmann-nebel-jair2001}~\cite{torralba-sievers-ijcai2019} with lazy-greedy search with and without preferred operators (p.o.). This latter configuration is equivalent to what is commonly known as FTSPlan~\cite{torralba-et-al-ipc2023}.
%SASE was not ran as its source code is not available. \alvaro{This sentence about SASE does not make too much sense to me. SASE appears out of the blue here, and with no citation. Why don't we mention PASAR or any other solver we cite in the introduction?}
Due to the large number of domains, we report only the total number of solved instances. An appendix is available
%~\cite{arxivappendix}
with detailed results per domain as well as a more in depth analysis of the impact of each individual projection.

\begin{table}
    \centering
    \small
    \begin{tabular}{r@{ }|r|r|r|r|r|r}
         & \multicolumn{3}{c|}{One-by-One} & \multicolumn{3}{c}{AlgC} \\ %\hline
         & Seq & S-L & Chains & Seq & S-L & Chains \\ \midrule
                 Baseline & 603 & -- & -- & 764 & -- & -- \\
                 %Baseline+P & 643 & -- & -- & 863 & -- & -- \\

        SL & 659 & 985 & 1096 & 1097 & 1270 & 1271 \\
        SL+LG & 668 & 1001 & 1116 & 1014 & 1221 & 1271  \\

         %P& LG & \textbf{712} & 1065 & 1184 & 1113 & 1336 & 1393 \\

         SL+P  & 708 & \textbf{1080} & \textbf{1196} & \textbf{1219} & \textbf{1418} & \textbf{1431} \\
         SL+P+LG & \textbf{712} & 1065 & 1184 & 1113 & 1336 & 1393 \\
         %---&LG & 668 & 1001 & 1116 & 1014 & 1221 & 1271  \\
    \end{tabular}
    \caption{Coverage for multiple configurations of formula generation and  parallelism (columns) and transition relation optimizations (rows).
    S-L and Chains require the SL optimisation.}
    \label{tab:res1}
\end{table}

\begin{table}
    \centering
    \small
    \setlength{\tabcolsep}{3pt}
%    \begin{tabular}{r|rrr|rr} 
%         & \multicolumn{3}{c|}{SAT} & &\\
%         & \multicolumn{1}{c}{Seq} & 
%         \multicolumn{1}{c}{Self-loop} & 
%         \multicolumn{1}{c|}{Chains} & 
%         \multicolumn{1}{c}{FF} & 
%         \multicolumn{1}{c}{MpC}\\\hline
%          & CMit_seq___ & CMit_slflpp & CMit_chains & ff \\
%None & (982)845 & 1146 & (1288)1214 & 1318 &  1253\\
%LR+S & (1239)1219 & 1418 & \textbf{1431} & 1430 & --\\
%LR+S+M & 767 & 1253 & 1216 & 1329  & -- \\
%          None & 404& 707& 1288& 1166\\
%          LR+shrink  & 666 & 943 & 1004 & 1389 \\
%          Merge & 207& 302& 276& -- 
%    \end{tabular}
    \begin{tabular}{r|rrr|rr|rr} 
    & \multicolumn{3}{c|}{SAT} &          \multicolumn{2}{c|}{FF} & \multicolumn{2}{c}{MpC}\\
    %\midrule
         & \multicolumn{1}{c}{Seq} & 
         \multicolumn{1}{c}{S-L} & 
         \multicolumn{1}{c|}{Chains} & 
         \multicolumn{1}{c}{--} & 
                  \multicolumn{1}{c|}{p.o.} & 
         \multicolumn{1}{c}{$\forall$} & \multicolumn{1}{c}{$\exists$}\\\midrule
% & CMit_seq___ & CMit_slflpp & CMit_chains & ff & ff-pref & MpC-A-C & MpC-E-C \\
None   & 979           & 1236          & 1288          & 1318          & 1615          & 1368 & 1431 \\
LR     & 1095          & 1338          & 1344          & 1324          & 1596          & --   & --   \\
LR+S   & \textbf{1219} & \textbf{1418} & \textbf{1431} & \textbf{1430} & \textbf{1627} & --   & --   \\
LR+S+M & 1039          & 1241          & 1210          & 1314          & 1537          & --   & --   \\
        \end{tabular}
    \caption{Coverage using different task transformations.}
    \label{tab:res2}
\end{table}

\paragraph{Parallelism and Formula Generation}
Table~\ref{tab:res1} highlights the performance differences across various encoding strategies. Sequential encodings significantly underperform encodings that support parallelism, solving around 300 fewer instances, since they require more timesteps to find a satisfying assignment. As a new formula is generated for each timestep, this results in more formulas being attempted and, consequently, an increase in required solving time. When comparing the two parallelism strategies, chain-based parallelism outperforms the self-loop-based approach, despite requiring additional variables and clauses. This indicates that the benefit of reducing the number of timesteps, and thus the number of formulas that must be solved, outweighs the cost introduced by the more complex encoding. Regarding formula generation strategies, encodings using the One-by-One approach significantly underperform, regardless of optimisations, compared to AlgC, solving 200–400 fewer instances. This is due to the requirement to prove unsatisfiability for each formula length before proceeding to the next, which becomes increasingly costly for longer plans.

\paragraph{Optimisations}
The projection optimisation, which reduces the number of constraints and their arity, shows a consistent improvement across configurations, solving 50 to 200 more instances. On the other hand, the label groups optimisation seems to have a negative impact. Aside from one result in One-by-One, the use of this optimisation always reduced the number of instances solved indicating that the introduction of the additional variable as well as the new overhead in the formula ends up being more harmful than beneficial. Regarding the Self-loop optimisation, we only present the baseline results without it because it would not be possible to allow any sort of relevant parallelism without this optimisation as otherwise the base constraint encoding the transitions would be too restrictive.
\begin{figure}[h!]   % (h!) = place HERE
    \centering
    \includegraphics[width=0.2\textwidth]{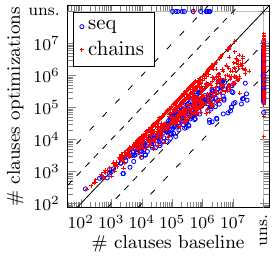} \quad
    \includegraphics[width=0.2\textwidth]{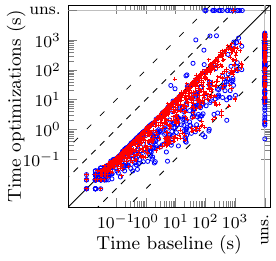}
    \caption{Number of clauses in first satisfiable formula and runtime of one-by-one with LR+S with/without projection optimizations.}
    \label{fig:ah}
\end{figure}

Figure~\ref{fig:ah} shows the difference in formula size (in terms of the number of clauses) and total time with and without using the projection optimisation. Using chains, with the optimisation enabled, the encoding produces less clauses for a total of 533 tasks while only exceeding the number of clauses for 299. Additionally, it can also be seen that the difference in the number of clauses in the cases where it produces less clauses is much bigger than in the cases where it produces more. We do not provide an analysis for the self-loop optimisation because it has to be enabled in order for parallelism to be allowed. We also do not provide a plot to analyse how the number of variables changes with each optimisation since the number of variables introduced by them is negligible.

\begin{figure}[h!]   % (h!) = place HERE
    \centering
    \includegraphics[width=0.2\textwidth]{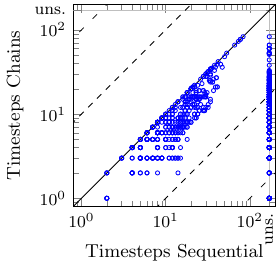} \quad
    \includegraphics[width=0.2\textwidth]{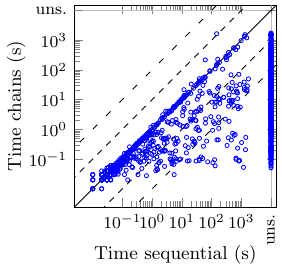}
    \caption{Minimum number of timesteps to find a plan and total time with chains parallelism (y-axis) versus without parallelism (x-axis). }
    \label{fig:parallelism}
\end{figure}
Figure~\ref{fig:parallelism} shows the impact of parallelism in more detail. In this case, the main difference is in reducing the number of time steps necessary to find a plan. This translates into a very significant difference in terms of time, sometimes of several orders of magnitude. This is not surprising because the SAT formula grows linearly with the number of time steps, but solving the larger formulas can be exponentially harder. 

\paragraph{Transformations}When analysing the impact of transformations, it is clear that they can influence performance, both positively and negatively. While our default set of transformations consistently achieves the best results, merge (M) transformations consistently decrease performance, sometimes even falling below the encoding without any transformations.  We evaluated the merge strategies available in FTSPlan, namely DFP and MIASM, setting a maximum factor size of 100. MIASM obtained slightly better results than DFP (around 5\% higher coverage) and is the configuration reported as M in the paper. Coverage was lower or equal in every single domain compared to encodings not using merging. This drop can be attributed to two key factors.

The first is the loss of parallelism. Merging transition systems can force previously independent labels to be treated as interdependent. For example, imagine two transition systems and two labels: one label is only relevant in the first system, and the other only in the second. After merging, both labels become relevant in the combined system and may even collapse into a single label during label reduction. As a result, they can no longer be executed in parallel, leading to longer plans with more timesteps.

The second issue is the increase in state space. When multiple systems are merged, the combined system may contain many more states than each system individually. This expansion leads to a corresponding increase in the number of possible transitions that must be encoded, thus increasing the overall size and complexity of the formula.

Notably, these merge strategies were originally devised for computing heuristics rather than for problem reformulation. In the former setting, merging typically continues until all variables have been merged, while for the latter setting it is arguably more important to decide if merging is useful at all. Other merge strategies from the literature are likely to run into the same issues for this reason. While setting a size limit, as done in FTSPlan and here, provides a practical safeguard, devising merge strategies specifically tailored to task reformulation remains an open problem and is beyond the scope of this paper.

\paragraph{FTS Benchmark}Table~\ref{tab:FTS_Benchmark} presents the results of running FF (with and without preferred operators) and our encodings on the FTS Benchmark using label reduction and shrinking. It can be seen that just like in the International Planning Competition Benchmark (IPC), the inclusion of the projections always has a positive impact. However, there are two key differences to the results of the IPC. The first is that now our encodings outperform FF both with and without preferred operators. The second is that the encoding using chains parallelism, achieves the worst results among our encodings. This is due to the fact that in this benchmark there is almost no parallelism allowed, and as such we cannot find significantly shorter plans using the encodings that support parallelism. So, in this scenario the chains encoding mostly only introduces the additional overhead of the extra variables and constraints, without any benefits. Note, however, that the lack of parallelism in this benchmark is not directly related to whether the input task is in \SAS or FTS representation, so one should not conclude that paralellism is irrelevant in this setting. This is simply due to it being a relatively small benchmark set compared to the IPC and consisting mostly of permutation puzzles who typically lack parallelism compared to planning tasks where multiple agents/components can perform actions simultaneously. 

% Nevertheless, we still believe that it is important to study how to encode parallelism, as this benchmark is still relatively small compared to the IPC, and it may grow and contain more instances that allow parallelism in the future.

\begin{table}[]
    \centering
    \small
    \begin{tabular}{l@{ }r|rr|rrr}
    & & \multicolumn{2}{c|}{FF} & \multicolumn{3}{c}{SAT}\\
 & $\Sigma$ & -- & (p.o.) & Chains & Seq & S-L \\ \midrule
burnt-pancakes & 100 & 60 & 63 & 48 & 54 & 55 \\
matrix-mult & 11 & 11 & 11 & 11 & 11 & 11 \\
cavediving-adl14 & 20 & 7 & 7 & 8 & 7 & 8 \\
pancakes & 100 & 58 & 66 & 53 & 58 & 63 \\
rubiks-cube & 100 & 9 & 10 & 45 & 47 & 47 \\
topspin & 100 & 22 & 22 & 29 & 30 & 28 \\ \hline \hline
$\Sigma$ & 431 & 167 & 179 & 194 & 207 & 212 \\
\end{tabular}
    \caption{FF and SAT on FTS Benchmark per Domain. Our encodings use AlgC and SL, and P.}
    \label{tab:FTS_Benchmark}
\end{table}
\section{Conclusions}

We presented several SAT encodings for Factored Transition Systems (FTS), including techniques leveraging projections, self-loops, label groups, and different forms of parallelism. Empirically, we have shown that our encoding matches the state-of-the-art SAT-based planner Madagascar in total instances solved in the IPC benchmark even though we do not support $\exists$-step parallelism. In the IPC benchmark, we also outperform FF without preferred operators while in the FTS benchmark, we outperform FF both with and without preferred operators.

For future work, we aim to further relax the constraints on parallelism, to support $\exists$-step parallelism. Additionally, exploring new merge strategies tailored for SAT encodings could help avoid the problems of existing merging strategies.

\section*{Acknowledgements}
This publication is part of the project “Exploiting Problem Structures in SAT-based Planning” with file number OCENW.M.22.050 of the research programme Open Competition which is financed by the Dutch Research Council (NWO).

\bibliographystyle{named}
\bibliography{abbrv,extra,ijcai26,crossref}

@String{aij = "Artificial Intelligence"}

@String{jair = "Journal of Artificial Intelligence Research"}

@String{jacm = "Journal of the ACM"}

@String{compint = "Computational Intelligence"}

@Proceedings{aaai1996,
  title =        "Proceedings of the Thirteenth National Conference on
                  Artificial Intelligence ({AAAI} 1996)",
  booktitle =    "Proceedings of the Thirteenth National Conference on
                  Artificial Intelligence ({AAAI} 1996)",
  publisher =    "{AAAI} Press",
  year =         "1996"
}

@Proceedings{aaai1997,
  title =        "Proceedings of the Fourteenth National Conference on
                  Artificial Intelligence ({AAAI} 1997)",
  booktitle =    "Proceedings of the Fourteenth National Conference on
                  Artificial Intelligence ({AAAI} 1997)",
  publisher =    "{AAAI} Press",
  year =         "1997"
}

@Proceedings{aaai2014,
  editor =       "Carla E. Brodley and Peter Stone",
  title =        "Proceedings of the Twenty-Eighth {AAAI} Conference on
                  Artificial Intelligence ({AAAI} 2014)",
  booktitle =    "Proceedings of the Twenty-Eighth {AAAI} Conference on
                  Artificial Intelligence ({AAAI} 2014)",
  publisher =    "{AAAI} Press",
  year =         "2014"
}

@Proceedings{ecai1992,
  editor =       "Bernd Neumann",
  title =        "Proceedings of the 10th {European} Conference on
                  {Artificial} {Intelligence} ({ECAI} 1992)",
  booktitle =    "Proceedings of the 10th {European} Conference on
                  {Artificial} {Intelligence} ({ECAI} 1992)",
  publisher =    "John Wiley and Sons",
  year =         "1992"
}

@Proceedings{ecai2014,
  editor =       "Torsten Schaub and Gerhard Friedrich and Barry O'Sullivan",
  title =        "Proceedings of the 21st {European} Conference on
                  {Artificial} {Intelligence} ({ECAI} 2014)",
  booktitle =    "Proceedings of the 21st {European} Conference on
                  {Artificial} {Intelligence} ({ECAI} 2014)",
  publisher =    "IOS Press",
  year =         "2014"
}

@Proceedings{icaps2009,
  editor =       "Alfonso Gerevini and Adele Howe and Amedeo Cesta and
                  Ioannis Refanidis",
  title =        "Proceedings of the Nineteenth International Conference on
                  Automated Planning and Scheduling (ICAPS 2009)",
  booktitle =    "Proceedings of the Nineteenth International Conference on
                  Automated Planning and Scheduling (ICAPS 2009)",
  year =         "2009",
  publisher =    "AAAI Press"
}

@Proceedings{icaps2016,
  editor =       "Amanda Coles and Andrew Coles and Stefan Edelkamp and Daniele Magazzeni and Scott Sanner",
  title =        "Proceedings of the Twenty-Sixth International Conference on
                  Automated Planning and Scheduling (ICAPS 2016)",
  booktitle =    "Proceedings of the Twenty-Sixth International Conference on
                  Automated Planning and Scheduling (ICAPS 2016)",
  year =         "2016",
  publisher =    "AAAI Press"
}

@Proceedings{icaps2023,
  editor =       "Sven Koenig and Roni Stern and Mauro Vallati",
  title =        "Proceedings of the Thirty-Third International Conference on
                  Automated Planning and Scheduling (ICAPS 2023)",
  booktitle =    "Proceedings of the Thirty-Third International Conference on
                  Automated Planning and Scheduling (ICAPS 2023)",
  year =         "2023",
  publisher =    "AAAI Press"
}

@Proceedings{icaps2024,
  editor =       "Sara Bernardini and Christian Muise",
  title =        "Proceedings of the Thirty-Fourth International Conference on
                  Automated Planning and Scheduling (ICAPS 2024)",
  booktitle =    "Proceedings of the Thirty-Fourth International Conference on
                  Automated Planning and Scheduling (ICAPS 2024)",
  year =         "2024",
  publisher =    "AAAI Press"
}

@Proceedings{ijcai2019,
  editor =       "Sarit Kraus",
  title =        "Proceedings of the 28th International Joint
                  Conference on Artificial Intelligence (IJCAI 2019)",
  booktitle =    "Proceedings of the 28th International Joint
                  Conference on Artificial Intelligence (IJCAI 2019)",
  publisher =    "IJCAI",
  year =         "2019"
}

@Proceedings{ipc2006,
  title =        "Fifth {I}nternational {P}lanning {C}ompetition ({IPC}-5):
                  Planner Abstracts",
  booktitle =    "Fifth {I}nternational {P}lanning {C}ompetition ({IPC}-5):
                  Planner Abstracts",
  year =         "2006"
}

@Proceedings{ipc2014,
  title =        "Eighth {I}nternational {P}lanning {C}ompetition ({IPC}-8):
                  Planner Abstracts",
  booktitle =    "Eighth {I}nternational {P}lanning {C}ompetition ({IPC}-8):
                  Planner Abstracts",
  year =         "2014"
}

@Proceedings{ipc2023,
  title =        "Tenth {I}nternational {P}lanning {C}ompetition ({IPC}-10):
                  Planner Abstracts",
  booktitle =    "Tenth {I}nternational {P}lanning {C}ompetition ({IPC}-10):
                  Planner Abstracts",
  year =         "2023"
}

@Proceedings{socs2014,
  editor =       "Stefan Edelkamp and Roman Bart\'{a}k",
  title =        "Proceedings of the Seventh Annual Symposium on
                  Combinatorial Search (SoCS 2014)",
  booktitle =    "Proceedings of the Seventh Annual Symposium on
                  Combinatorial Search (SoCS 2014)",
  publisher =    "AAAI Press",
  year =         "2014"
}

@Proceedings{socs2018,
  editor =       "Vadim Bulitko and Sabine Storandt",
  title =        "Proceedings of the 11th Annual Symposium on
                  Combinatorial Search (SoCS 2018)",
  booktitle =    "Proceedings of the 11th Annual Symposium on
                  Combinatorial Search (SoCS 2018)",
  publisher =    "AAAI Press",
  year =         "2018"
}

@Proceedings{socs2019,
  editor =       "Pavel Surynek and William Yeoh",
  title =        "Proceedings of the 12th Annual Symposium on
                  Combinatorial Search (SoCS 2019)",
  booktitle =    "Proceedings of the 12th Annual Symposium on
                  Combinatorial Search (SoCS 2019)",
  publisher =    "AAAI Press",
  year =         "2019"
}

@inproceedings{sinz2005towards,
                Title={Towards an Optimal {CNF} Encoding of Boolean Cardinality Constrai
nts},
                Author={Sinz, Carsten},
                Booktitle = {Proceedings of the 11th International Conference on Principles and Practice of Constraint Programming ({CP} 2005)},
                Volume={3709},
                Pages={827--831},
                Year={2005},
                Publisher={Springer}
}

@inproceedings{BiereFallerFazekasFleuryFroleyksPollitt-SAT-Competition-2024-solvers,
  author       = {Armin Biere and Tobias Faller and Katalin Fazekas and Mathias Fleury and Nils Froleyks and Florian Pollitt},
  title	       = {{CaDiCaL}, {Gimsatul}, {IsaSAT} and {Kissat} Entering the {SAT Competition 2024}},
  editor       = {Marijn Heule and Markus Iser and Matti J{\"a}rvisalo and Martin Suda},
  booktitle    = {Proc.~of {SAT Competition} 2024 -- Solver, Benchmark and
                  Proof Checker Descriptions},
  volume       = {B-2024-1},
  series       = {Department of Computer Science Report Series B},
  publisher    = {University of Helsinki},
  year	       = 2024,
  pages	       = {8-10},
}

@Article{backstrom-nebel-compint1995,
  author =       "Christer B{\"a}ckstr{\"o}m and Bernhard Nebel",
  title =        "Complexity Results for {SAS$^{+}$} Planning",
  journal =      compint,
  year =         "1995",
  volume =       "11",
  number =       "4",
  pages =        "625--655"
}

@InProceedings{buechner-et-al-icaps2024,
  author =       "Clemens B{\"u}chner and Patrick Ferber and Jendrik Seipp and Malte Helmert",
  title =        "Abstraction Heuristics for Factored Tasks",
  crossref =     "icaps2024",
  pages =        "40--49"
}

@Article{domshlak-et-al-jair2009,
  author =       "Carmel Domshlak and J{\"o}rg Hoffmann and
                  Ashish Sabharwal",
  title =        "Friends or Foes? {On} Planning as Satisfiability and
                  Abstract {CNF} Encodings",
  journal =      jair,
  volume =       "36",
  year =         "2009",
  pages =        "415--469"
}

@InProceedings{fan-et-al-socs2014,
  author =       "Gaojian Fan and Martin M{\"{u}}ller and Robert Holte",
  title =        "Non-Linear Merging Strategies for Merge-and-Shrink Based on Variable Interactions",
  crossref =     "socs2014",
  pages =        "53--61"
}

@Article{fikes-nilsson-aij1971,
  author =       "Richard E. Fikes and Nils J. Nilsson",
  title =        "{STRIPS}: {A} New Approach to the Application of
                  Theorem Proving to Problem Solving",
  journal =      aij,
  year =         "1971",
  volume =       "2",
  pages =        "189--208"
}

@InProceedings{froleyks-et-al-socs2019,
  author =       "Nils Christian Froleyks and
                  Tom{\'{a}}s Balyo and
                  Dominik Schreiber",
  title =        "{PASAR} - Planning as Satisfiability with Abstraction Refinement",
  pages =        "70--78",
  crossref =     "socs2019",
}

@Article{helmert-et-al-jacm2014,
  author =       "Malte Helmert and Patrik Haslum and {J\"org} Hoffmann
                  and Raz Nissim",
  title =        "Merge-and-Shrink Abstraction: A Method for Generating
                  Lower Bounds in Factored State Spaces",
  journal =      jacm,
  year =         "2014",
  volume =       "61",
  number =       "3",
  pages =        "16:1--63"
}

@Article{helmert-jair2006,
  author =       "Malte Helmert",
  title =        "The {Fast} {Downward} Planning System",
  journal =      jair,
  year =         "2006",
  volume =       "26",
  pages =        "191--246"
}

@InProceedings{hoffmann-et-al-ecai2014,
  author =       "J{\"o}rg Hoffmann and Peter Kissmann and {\'A}lvaro Torralba",
  title =        "``{D}istance''? {W}ho Cares? {T}ailoring
                  Merge-and-Shrink Heuristics to Detect Unsolvability",
  pages =        "441--446",
  crossref =     "ecai2014"
}

@Article{hoffmann-nebel-jair2001,
  author =       "J{\"o}rg Hoffmann and Bernhard Nebel",
  title =        "The {FF} Planning System: {Fast} Plan Generation Through
                  Heuristic Search",
  journal =      jair,
  year =         "2001",
  volume =       "14",
  pages =        "253--302"
}

@Article{huang-et-al-jair2012,
  author =       "Ruoyun Huang and Yixin Chen and Weixiong Zhang",
  title =        "{SAS}$^+$ Planning as Satisfiability",
  journal =      jair,
  volume =       "43",
  pages =        "293--328",
  year =         "2012"
}

@InProceedings{kautz-et-al-ipc2006,
  author =       "Henry Kautz and Bart Selman and J{\"o}rg Hoffmann",
  title =        "{SatPlan}: Planning as Satisfiability",
  crossref =     "ipc2006"
}

@InProceedings{kautz-selman-aaai1996,
  author =       "Henry Kautz and Bart Selman",
  title =        "Pushing the Envelope: Planning, Propositional Logic,
                  and Stochastic Search",
  pages =        "1194--1201",
  crossref =     "aaai1996"
}

@InProceedings{kautz-selman-ecai1992,
  author =       "Henry Kautz and Bart Selman",
  title =        "Planning as Satisfiability",
  pages =        "359--363",
  crossref =     "ecai1992"
}

@InProceedings{korf-aaai1997,
  author =       "Richard E. Korf",
  title =        "Finding Optimal Solutions to {Rubik}'s {Cube} Using
                  Pattern Databases",
  pages =        "700--705",
  crossref =     "aaai1997"
}

@Article{nebel-jair2000,
  author =       "Bernhard Nebel",
  title =        "On the Compilability and Expressive Power of
                  Propositional Planning Formalisms",
  journal =      jair,
  year =         "2000",
  volume =       "12",
  pages =        "271--315"
}

@Article{rintanen-aij2012,
  author =       "Jussi Rintanen",
  title =        "Planning as Satisfiability: Heuristics",
  journal =      aij,
  volume =       "193",
  year =         "2012",
  pages =        "45--86"
}

@Article{rintanen-et-al-aij2006,
  author =       "Jussi Rintanen and Keijo Heljanko and Ilkka Niemel{\"a}",
  title =        "Planning as satisfiability: parallel plans and
                  algorithms for plan search",
  journal =      aij,
  volume =       "170",
  number =       "12--13",
  year =         "2006",
  pages =        "1031--1080"
}

@InProceedings{rintanen-ipc2011,
  author =       "Jussi Rintanen",
  title =        "Madagascar: Scalable Planning with {SAT}",
  booktitle =    "IPC 2011 Planner Abstracts",
  year =         "2011",
  pages =        "61--64"
}

@InProceedings{rintanen-ipc2014,
  author =       "Jussi Rintanen",
  title =        "Madagascar: Scalable Planning with {SAT}",
  crossref =     "ipc2014",
  pages =        "66--70"
}

@InProceedings{robinson-et-al-icaps2009,
  author =       "Nathan Robinson and Charles Gretton and Duc Nghia Pham and Abdul Sattar",
  title =        "{SAT}-Based Parallel Planning Using a Split Representation
                  of Actions",
  crossref =     "icaps2009",
  pages =        "281--288"
}

@Misc{seipp-et-al-zenodo2017,
  author =       "Jendrik Seipp and Florian Pommerening and
                  Silvan Sievers and Malte Helmert",
  title =        "{Downward} {Lab}",
  publisher =    "Zenodo",
  year =         "2017",
  howpublished = "\url{https://doi.org/10.5281/zenodo.790461}"
}

@InProceedings{sievers-et-al-aaai2014,
  author =       "Silvan Sievers and Martin Wehrle and Malte Helmert",
  title =        "Generalized Label Reduction for Merge-and-Shrink Heuristics",
  crossref =     "aaai2014",
  pages =        "2358--2366"
}

@InProceedings{sievers-et-al-icaps2016,
  author =       "Silvan Sievers and Martin Wehrle and Malte Helmert",
  title =        "An Analysis of Merge Strategies for Merge-and-Shrink Heuristics",
  pages =        "294--298",
  crossref =     "icaps2016"
}

@InProceedings{sievers-et-al-icaps2024,
  author =       "Silvan Sievers and Thomas Keller and Gabriele R{\"o}ger",
  title =        "Merging or Computing Saturated Cost Partitionings? A Merge Strategy for the Merge-and-Shrink Framework",
  crossref =     "icaps2024",
  pages =        "541--545"
}

@Article{sievers-helmert-jair2021,
  author =       "Silvan Sievers and Malte Helmert",
  title =        "Merge-and-Shrink: A Compositional Theory of Transformations of Factored Transition Systems",
  journal =      jair,
  year =         "2021",
  volume =       "71",
  pages =        "781--883"
}

@InProceedings{sievers-socs2018,
  author =       "Silvan Sievers",
  title =        "Merge-and-Shrink Heuristics for Classical Planning:
                  Efficient Implementation and Partial Abstractions",
  pages =        "90--98",
  crossref =     "socs2018"
}

@InProceedings{speck-et-al-icaps2023,
  author =       "David Speck and Paul H{\"o}ft and Daniel Gnad and Jendrik Seipp",
  title =        "Finding Matrix Multiplication Algorithms with Classical Planning",
  crossref =     "icaps2023",
  pages =        "411--416"
}

@Article{torralba-et-al-aij2017,
  author =       "{\'A}lvaro Torralba and Vidal Alc\'{a}zar and Peter Kissmann and Stefan Edelkamp",
  title =        "Efficient Symbolic Search for Cost-optimal Planning",
  journal =      aij,
  volume =       "242",
  year =         "2017",
  pages =        "52--79"
}

@InProceedings{torralba-et-al-ipc2023,
  author =       "{\'{A}}lvaro Torralba and Silvan Sievers and
                  Rasmus G. Tollund and Kristian \O.\ Nielsen",
  title =        "{FTSPlan}: Task Reformulation via Merge-and-Shrink",
  crossref =     "ipc2023"
}

@InProceedings{torralba-sievers-ijcai2019,
  author =       "{\'A}lvaro Torralba and Silvan Sievers",
  title =        "Merge-and-Shrink Task Reformulation for Classical Planning",
  crossref =     "ijcai2019",
  pages =        "5644--5652",
}

@inproceedings{behnke2025axsat,
  title={AxSAT--bringing axioms to sat planning},
  author={Behnke, Gregor and Speck, David and Gnad, Daniel},
  booktitle={European Conference on Logics in Artificial Intelligence},
  pages={77--93},
  year={2025},
  organization={Springer}
}

\newpage

\section{Appendix}

\subsection{Examples}

\begin{example}

To see why moving from primal to dual constraints is not sound anymore when using chains parallelism, consider the following: Let $\factor{1}$ and $\factor{2}$ be 2 different transition systems in which we will only look at two states of each, A and B, and C and D respectively and let $l_1$ and $l_2$ be two labels whose transitions are represented in the adjacency matrices.

\noindent
\begin{minipage}[t]{0.2\columnwidth}
  \centering
  \addtolength{\tabcolsep}{-0.4em}
    \begin{tabular}{c|c|c|c|}
         $l_1$&  A'&  B'\\\hline
         A&  0&  1\\\hline
         B&  0&  0\\\hline
    \end{tabular}
    \captionof{table}{$l_1$}
    \label{tab5:T_1_l1_adjacency_matrix}
\end{minipage}%
\hfill
\begin{minipage}[t]{0.2\columnwidth}
    \centering
    \addtolength{\tabcolsep}{-0.4em}
    \begin{tabular}{c|c|c|c|}
         $l_2$&  A'&  B'\\\hline
         A&  0&  0\\\hline
         B&  0&  0\\\hline
    \end{tabular}
    \captionof{table}{$l_2$}
    \label{tab5:T_1_l2_adjacency_matrix}
\end{minipage}
\hfill
\begin{minipage}[t]{0.2\columnwidth}
    \centering
    \addtolength{\tabcolsep}{-0.4em}
    \begin{tabular}{c|c|c|c|}
         $l_1$&  C'&  D'\\\hline
         C&  1&  0\\\hline
         D&  0&  1\\\hline
    \end{tabular}
    \captionof{table}{$l_1$}
    \label{tab5:T_2_l1_adjacency_matrix}
\end{minipage}
\hfill
\begin{minipage}[t]{0.2\columnwidth}
    \centering
    \addtolength{\tabcolsep}{-0.4em}
    \begin{tabular}{c|c|c|c|}
         $l_2$&  C'&  D'\\\hline
         C&  0&  1\\\hline
         D&  0&  0\\\hline
    \end{tabular}
    \captionof{table}{$l_2$}
    \label{tab5:T_2_l2_adjacency_matrix}
\end{minipage}

By using dual (projected) constraints, we would get the following constraints for $\factor{1}$: $A \implies l_1$ and $B' \implies l_1$ and no constraint would be generated involving $l_2$ since it has no transitions in these states and we are only considering dual constraints. We would also have constraints such as $A \implies B'$, however these are irrelevant for this example and are omitted for brevity. For $\factor{2}$ there would be no constraint for $l_1$ since it is an irrelevant label, and similarly to $\factor{1}$ we would have the constraints $C \implies l_2$ and $D' \implies l_2$. From these constraints alone, there is nothing forbidding both labels from being selected and having $\factor{1}$ and $\factor{2}$ transitioning from $A$ to $B'$ and from $C$ to $D'$ respectively. However, it should be clear to see that in the synchronised product of the transitions systems it is not possible for both of these labels to be executed. This type of invalid parallelism also is not automatically forbidden by the chains because the only variables that are considered in each chain are those of the labels which transition to a specific target, instead of considering all labels as this latter option would lead to scalability issues.

\end{example}

\subsection{Experiments}

\subsubsection{Impact of individual projections}

As mentioned in the main paper, in total, there can be 3 projections: projection into the (label,source) plane, which we call "rows"(R), projection into the (label, target) plane, which we call "columns"(C), and projection into the (source, target) plane, which we call "pillars"(P) and, as described in constraint(10), the constraints for these projections are written as $A \implies \neg C'$ for all impossible pairs in the projection matrix. We chose to encode it this way to keep the arity of the clause as small as possible. However, if in a line of the projection matrix, there exists only one(O) 1, then it is better to only encode the positive constraint, e.g., instead of encoding $A \implies \neg A'$ and $A \implies \neg C'$, we could simply encode $A \implies B'$, as the arity remains the same but the number of constraints is reduced. The same type of logic can be applied outside projections, where constraints(8) or (9) would be used as a last resort(L), to encode information not covered by the constraints derived from the projections.

In Table~\ref{tab:res_proj}, we report the results obtained using each projection optimisation individually, as well as all optimisations combined, both with no task transformation and with the default task transformation. 
Additionally, we report the minimum and maximum improvement achieved by each optimisation. Improvements are computed as the difference, for each column, between an encoding with a given projection optimisation and the corresponding encoding without that optimisation. For example, the improvement of (SL,-,\_\_P\_L) is in regard to (SL,-,\_\_\_\_L), and the improvement of  (SL,-,\_\_POL) is in regard to (SL,-,\_\_P\_L). It can be seen that between RCP, R and C yield similar improvements, which are more significant than the improvement provided by P. This is because P focuses only on states while R and C focus on both states and labels. Since in the transition systems it is not overwhelmingly common for a state to only transition to another state or to move to very few states, the information gained form using P is not as big/useful as from R or C. Regarding O, it can also be seen that it does not yield the most significant improvement, but this is to be expected as it is only a simplification of the other projections, helping reduce formula size, it does not actually provide new information to help find a solution faster.

\begin{table}
    \centering
    \small
    \begin{tabular}{c|c|c|c|c||c|c}
        & \multicolumn{2}{c|}{One-by-One} & \multicolumn{2}{c||}{AlgC} & \multicolumn{2}{c}{Improvement}\\ \hline
        & None & LR+S & None & LR+S & Min & Max \\ \hline \hline
        (SL,-,\_\_\_\_L) & 854 & 1112 & 915 & 1288 & - & - \\
        (SL,-,\_\_P\_L) & 917 & 1128 & 1010 & 1326 & 16 & 95 \\
        (SL,-,\_\_POL) & 919 & 1132 & 1010 & 1328 & 2 & 4 \\
        (SL,-,\_C\_\_L) & 972 & 1147 & 1113 & 1334 & 35 & 198 \\
        (SL,-,\_C\_OL) & 985 & 1157 & 1153 & 1350 & 10 & 40 \\
        (SL,-,R\_\_\_L) & 978 & 1152 & 1171 & 1349 & 40 & 256 \\
        (SL,-,R\_\_OL) & 1002 & 1170 & 1232 & 1379 & 18 & 61 \\
        (SL,-,RCP\_L) & 995 & 1166 & 1171 & 1373 & 54 & 256 \\
        (SL,-,RCPOL) & 1041 & 1196 & 1288 & 1431 & 30 & 117 \\
        (SL,LG,\_\_\_\_L) & 913 & 1122 & 956 & 1295 & - & - \\
        (SL,LG,\_\_P\_L) & 955 & 1144 & 1016 & 1314 & 19 & 60 \\
        (SL,LG,\_\_POL) & 956 & 1148 & 1027 & 1330 & 1 & 16 \\
        (SL,LG,\_C\_\_L) & 977 & 1152 & 1083 & 1326 & 30 & 127 \\
        (SL,LG,\_C\_OL) & 988 & 1158 & 1110 & 1340 & 6 & 27 \\
        (SL,LG,R\_\_\_L) & 986 & 1156 & 1123 & 1338 & 34 & 167 \\
        (SL,LG,R\_\_OL) & 1000 & 1165 & 1160 & 1351 & 9 & 37 \\
        (SL,LG,RCP\_L) & 996 & 1162 & 1143 & 1368 & 40 & 187 \\
        (SL,LG,RCPOL) & 1027 & 1184 & 1214 & 1393 & 22 & 71 \\
    \end{tabular}
    \caption{Projections impact and improvement gained from each projection(using chains parallelism)}
    \label{tab:res_proj}
\end{table}

% \subsection{Impact of Label Groups}

% Table~\ref{tab:res_with_lg} extends the results of Table~4 in the main paper to include the Label Group (LG) optimisation. However, this optimisation, seems to have a negative impact. Aside from some results in One-by-One, the use of this optimisation always reduced the number of instances solved. This leads us to believe that the introduction of the additional variable as well as the new overhead in the formula ends up being more harmful than beneficial.

% \begin{table}
%     \centering
%     \small
%     \begin{tabular}{l@{ }l|c|c|c|c|c|c}
%          && \multicolumn{3}{c|}{One-by-One} & \multicolumn{3}{c}{AlgC} \\ %\hline
%          && Seq & S-L & Chains & Seq & S-L & Chains \\ \midrule
%          P& LG & \textbf{712} & 1065 & 1184 & 1113 & 1336 & 1393 \\
%          P&---  & 708 & \textbf{1080} & \textbf{1196} & \textbf{1219} & \textbf{1418} & \textbf{1431} \\
%          ---&LG & 668 & 1001 & 1116 & 1014 & 1221 & 1271  \\
%          ---& ---& 659 & 985 & 1096 & 1097 & 1270 & 1271 \\
%     \end{tabular}
%     \caption{Coverage for multiple configurations of formula generation (one-by-one and AlgC), parallelism (sequential, self-loop, and chains)  and optimisations (projections and label groups).}
%     \label{tab:res_with_lg}
% \end{table}

\subsubsection{Impact of Self-Loop optimisation on sequential encoding}

Since the use of the self-loop optimisation is mandatory to allow any sort of relevant parallelism, we can only study its impact on the sequential encoding. Table~\ref{tab:sl_impact} presents the results obtained from running different combinations of the three optimisations together with the default task transformations and no transformations as well as the 2 different formula generation strategies. It can be seen that in every single case, the enabling of the self-loop optimisation increases the number of solved instances. With the One-by-One strategy, it increases the number of solved instances by around 60 on average and with AlgC, it increases the number of solved instances by around 300 on average. This indicates that this optimisation has a positive impact regardless of the situation, unlike the label-group optimisation which only slightly improved results in the most basic encodings. 

\begin{table}
    \centering
    \begin{tabular}{c|c|c|c|c}
        & \multicolumn{2}{c|}{One-by-One} & \multicolumn{2}{c}{AlgC} \\ \hline
        & LR+S & None & LR+S & None \\ \hline \hline
        (-,-,-) & 603 & 351 & 764 & 451 \\
        (-,-,\_\_\_\_L) & 620 & 370 & 825 & 482 \\
        (-,-,RCPOL) & 643 & 396 & 863 & 520 \\
        (-,-,RCPO\_) & 641 & 393 & 841 & 514 \\
        (-,LG,-)& 631 & 394 & 820 & 480 \\
        (-,LG,\_\_\_\_L)& 636 & 395 & 809 & 472 \\
        (-,LG,RCPOL)& 648 & 418 & 832 & 497 \\
        (-,LG,RCPO\_)& 647 & 418 & 841 & 509 \\
        (SL,-,-) & 659 & 408 & 1097 & 678 \\
        (SL,-,\_\_\_\_L) & 673 & 420 & 1115 & 679 \\
        (SL,-,RCPO\_) & 708 & 480 & 1239 & 982 \\
        (SL,-,RCPOL) & 708 & 480 & 1219 & 979 \\
        (SL,LG,-)& 668 & 439 & 1014 & 650 \\
        (SL,LG,\_\_\_\_L)& 673 & 444 & 1021 & 661 \\
        (SL,LG,RCPO\_)& 712 & 481 & 1122 & 847 \\
        (SL,LG,RCPOL)& 712 & 481 & 1113 & 847 \\
    \end{tabular}
    \caption{Optimisations impact (using no parallelism)}
    \label{tab:sl_impact}
\end{table}

% \subsection{FTS Benchmark}

% Table~\ref{tab:FTS_Benchmark} presents the results of running FF (with and without preferred operators) and our encodings on the FTS Benchmark. It can be seen that just like in the International Planning Competition Benchmark (IPC), the inclusion of the projections always has a positive impact. However, there are 2 key differences to the results of the IPC. The first is that now our encodings outperform FF both with and without preferred operators. The second is that the encoding using chains parallelism, achieves the worst results between our encodings. This is due to the fact that in this benchmark there is almost no parallelism allowed, and as such we cannot find significant shorter plans using an encoding supporting parallelism. So, in this scenario the chains encoding mostly only introduces the additional overhead of the extra variables and constraints, it does not actually give significant benefits. Nevertheless, we still believe that it is important to study how to encode parallelism, as this benchmark is still relatively small compared to the IPC, and it may grow and contain more instances that allow parallelism in the future.

\begin{table*}[]
    \centering
    \small
    \begin{tabular}{cr|rr|rr|rr|rrr}
    & & \multicolumn{2}{c|}{FF} & \multicolumn{7}{c}{SAT}\\
 & $\Sigma$ & -- & (p.o.) & Chains + SL + P & Chains +SL & S-L + SL + P & S-L + SL & Seq+P+SL & Seq +SL & Seq \\ \midrule
pancakes & 100 & 58 & 66 & 53 & 49 & 63 & 51 & 58 & 53 & 53 \\
cavediving-adl14 & 20 & 7 & 7 & 8 & 8 & 8 & 8 & 7 & 7 & 7 \\
rubiks-cube & 100 & 9 & 10 & 45 & 43 & 47 & 42 & 47 & 44 & 43 \\
topspin & 100 & 22 & 22 & 29 & 29 & 28 & 30 & 30 & 29 & 28 \\
burnt-pancakes & 100 & 60 & 63 & 48 & 41 & 55 & 49 & 54 & 46 & 49 \\
matrix-multiplication & 11 & 11 & 11 & 11 & 11 & 11 & 11 & 11 & 11 & 11 \\ \midrule
$\Sigma$ & 431 & 167 & 179 & 194 & 181 & 212 & 191 & 207 & 190 & 191 \\
\end{tabular}
    \caption{FF and SAT on FTS Benchmark per Domain. Our encodings use AlgC. Optimisations are indicated. P abbreviates RCPOL.}
    \label{tab:FTS_Benchmark}
\end{table*}

% \begin{table*}[]
%     \centering
%     \small
%     \begin{tabular}{cr|rr|rrr}
%     & & \multicolumn{2}{c|}{FF} & \multicolumn{3}{c}{SAT}\\
%  & $\Sigma$ & -- & (p.o.) & Chains & Seq & S-L \\ \midrule
% burnt-pancakes & 100 & 60 & 63 & 48 & 54 & 55 \\
% matrix-multiplication & 11 & 11 & 11 & 11 & 11 & 11 \\
% cavediving-adl14 & 20 & 7 & 7 & 8 & 7 & 8 \\
% pancakes & 100 & 58 & 66 & 53 & 58 & 63 \\
% rubiks-cube & 100 & 9 & 10 & 45 & 47 & 47 \\
% topspin & 100 & 22 & 22 & 29 & 30 & 28 \\ \hline \hline
% $\Sigma$ & 431 & 167 & 179 & 194 & 207 & 212 \\
% \end{tabular}
%     \caption{FF and SAT on FTS Benchmark per Domain. Our encodings use AlgC and SL, and RCPOL.}
%     \label{tab:placeholder}
% \end{table*}

\begin{table*}[]
    \centering
    \small
    \begin{tabular}{c|rrr|rr}
    & \multicolumn{3}{c|}{SAT} & \multicolumn{2}{c}{FF}\\
 & Seq & S-L & Chains & -- & p.o. \\ \midrule
None & 203 & 208 & 190 & 164 & 179 \\
LR & 203 & 209 & 191 & 164 & 176 \\
LR+S & 207 & 212 & 194 & 167 & 179 \\
LR+S+M & 202 & 208 & 191 & 167 & 179 \\
\end{tabular}
    \caption{FTS Benchmark. Solved instances with transformations.\\
    All our SAT configurations use AlgC, SL, and RCPOL.}
    \label{tab:placeholder}
\end{table*}

\begin{table*}[]
    \centering
    \small
    \begin{tabular}{c|rrr|rr}
    & \multicolumn{3}{c|}{SAT} & \multicolumn{2}{c}{FF}\\
 & Seq & S-L & Chains & -- & p.o. \\ \midrule
None & 186 & 187 & 178 & 164 & 179 \\
LR & 186 & 189 & 178 & 164 & 176 \\
LR+S & 190 & 191 & 181 & 167 & 179 \\
LR+S+M & 182 & 185 & 172 & 167 & 179 \\
\end{tabular}
    \caption{FTS Benchmark. Solved instances with transformations but without projection optimisations.\\
    All our SAT configurations use AlgC, SL. RCPOL is deactivated.}
    \label{tab:placeholder}
\end{table*}

\subsubsection{Other tables and plots}

\begin{table*}[]
    \centering
    \small
    \begin{tabular}{cr|rrr|rr|rr}
    & & \multicolumn{3}{c|}{SAT} & \multicolumn{2}{c|}{FF} & \multicolumn{2}{c}{MpC} \\
 & $\Sigma$ & Chains & Seq & S-L & -- & (p.o.) & $\forall$ & $\exists$ \\ \midrule
movie & 30 & 30 & 30 & 30 & 30 & 30 & 30 & 30 \\
data-network & 20 & 5 & 2 & 4 & 5 & 11 & 5 & 1 \\
parcprinter & 50 & 50 & 49 & 50 & 50 & 50 & 50 & 50 \\
zenotravel & 20 & 20 & 20 & 20 & 20 & 20 & 20 & 20 \\
pipesworld-notankage & 50 & 41 & 20 & 39 & 28 & 40 & 42 & 38 \\
satellite & 36 & 28 & 27 & 28 & 28 & 28 & 28 & 29 \\
childsnack & 20 & 20 & 14 & 20 & 1 & 8 & 13 & 6 \\
pegsol & 50 & 50 & 50 & 50 & 50 & 50 & 50 & 50 \\
tetris & 20 & 3 & 0 & 3 & 3 & 2 & 2 & 3 \\
nomystery & 20 & 19 & 13 & 20 & 10 & 10 & 5 & 13 \\
miconic & 150 & 150 & 150 & 150 & 150 & 150 & 150 & 150 \\
tpp & 30 & 30 & 28 & 30 & 30 & 30 & 27 & 28 \\
freecell & 80 & 65 & 50 & 64 & 80 & 80 & 44 & 45 \\
parking & 40 & 0 & 0 & 0 & 12 & 17 & 3 & 16 \\
airport & 50 & 38 & 22 & 32 & 39 & 38 & 46 & 46 \\
depot & 22 & 20 & 8 & 19 & 20 & 22 & 21 & 22 \\
psr-small & 50 & 50 & 50 & 50 & 50 & 50 & 49 & 50 \\
barman & 40 & 0 & 0 & 0 & 12 & 12 & 8 & 19 \\
rovers & 40 & 40 & 39 & 40 & 40 & 40 & 40 & 40 \\
grid & 5 & 4 & 4 & 4 & 5 & 5 & 5 & 5 \\
trucks & 30 & 21 & 14 & 17 & 16 & 15 & 6 & 12 \\
ricochet-robots-sat23-adl & 20 & 0 & 0 & 0 & 5 & 7 & 0 & 0 \\
mystery & 30 & 19 & 19 & 19 & 19 & 18 & 18 & 19 \\
scanalyzer & 50 & 38 & 23 & 37 & 50 & 50 & 44 & 40 \\
tidybot & 20 & 0 & 2 & 0 & 13 & 14 & 13 & 17 \\
snake & 20 & 1 & 1 & 0 & 5 & 5 & 6 & 6 \\
agricola & 20 & 3 & 2 & 4 & 1 & 3 & 0 & 0 \\
floortile & 40 & 36 & 8 & 36 & 8 & 8 & 40 & 40 \\
gripper & 20 & 20 & 20 & 20 & 20 & 20 & 20 & 20 \\
logistics & 63 & 63 & 63 & 63 & 63 & 63 & 56 & 59 \\
openstacks-08-11-14 & 100 & 30 & 26 & 41 & 76 & 74 & 57 & 53 \\
pipesworld-tankage & 50 & 25 & 9 & 21 & 21 & 39 & 22 & 14 \\
quantum-layout-sat23 & 20 & 14 & 6 & 12 & 18 & 20 & 18 & 19 \\
sokoban & 50 & 8 & 5 & 8 & 48 & 48 & 8 & 8 \\
storage & 30 & 18 & 16 & 18 & 18 & 20 & 24 & 29 \\
pathways & 30 & 30 & 29 & 30 & 9 & 23 & 30 & 30 \\
spider & 20 & 0 & 0 & 0 & 7 & 7 & 0 & 0 \\
organic-synthesis & 20 & 2 & 2 & 2 & 2 & 2 & 0 & 0 \\
driverlog & 20 & 20 & 15 & 20 & 20 & 20 & 20 & 20 \\
woodworking & 50 & 50 & 48 & 50 & 28 & 50 & 50 & 50 \\
blocks & 35 & 19 & 18 & 19 & 35 & 35 & 35 & 35 \\
transport & 70 & 70 & 70 & 70 & 70 & 70 & 16 & 21 \\
hiking & 20 & 20 & 4 & 16 & 20 & 19 & 4 & 3 \\
elevators & 50 & 48 & 34 & 50 & 50 & 50 & 19 & 47 \\
mprime & 35 & 35 & 35 & 35 & 30 & 35 & 31 & 34 \\
visitall & 40 & 1 & 0 & 0 & 16 & 17 & 7 & 7 \\
maintenance-adl & 20 & 20 & 17 & 20 & 1 & 1 & 14 & 14 \\
thoughtful & 20 & 5 & 5 & 5 & 11 & 16 & 5 & 5 \\
ged & 20 & 0 & 0 & 0 & 20 & 20 & 13 & 14 \\
organic-synthesis-split & 20 & 2 & 2 & 2 & 3 & 3 & 4 & 4 \\
termes & 20 & 0 & 0 & 0 & 14 & 13 & 0 & 0 \\
schedule & 150 & 150 & 150 & 150 & 50 & 149 & 150 & 150 \\ \hline \hline
$\Sigma$ & 2026 & 1431 & 1219 & 1418 & 1430 & 1627 & 1368 & 1431 \\
 \end{tabular}
    \caption{Per Domain comparison of configurations.\\
    All SAT configurations use AlgC, and our configurations use SL and RCPOL.}
    \label{tab:placeholder}
\end{table*}

\begin{table*}[]
    \centering
    \small
    \begin{tabular}{c|rrrrrrrrrrrr}
 & Seq & S-L & Chains & FF & FF(p.o.) & MpC-$\forall$ & MpC-$\exists$ \\ \hline \hline
Seq & -- & 5 & 3 & 186 & 40 & 118 & 89 \\
S-L & 204 & -- & 20 & 252 & 88 & 188 & 146 \\
Chains & 215 & 33 & -- & 259 & 91 & 195 & 157 \\
FF & 397 & 264 & 258 & -- & 13 & 307 & 262 \\
FF(p.o.) & 448 & 297 & 287 & 210 & -- & 339 & 296 \\
MpC-$\forall$ & 267 & 138 & 132 & 245 & 80 & -- & 44 \\
MpC-$\exists$ & 301 & 159 & 157 & 263 & 100 & 107 & -- \\    \end{tabular}
    \caption{Per Instance compasision of configurations. In each row, we count how many instances that configuration solved, which were not solved by the configuration in the column. For example, there were 5 instances that Seq solved, but S-L did not.\\
    All SAT configurations use AlgC, and our configurations use SL and RCPOL.}
    \label{tab:placeholder}
\end{table*}

\begin{figure*}
\centering
\input{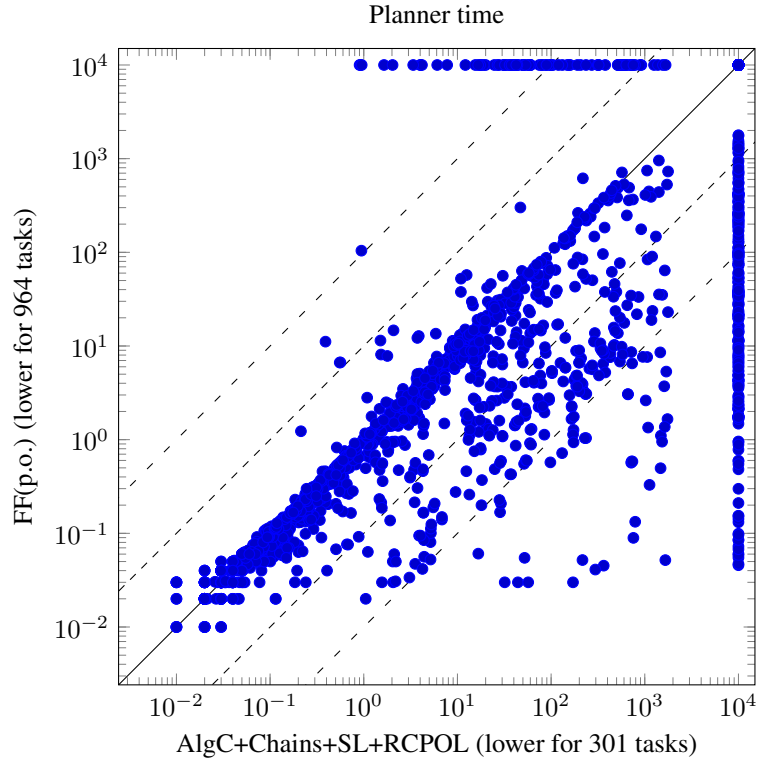}
\caption{Planner time comparison between FF(p.o.) and our best encoding}
\end{figure*}

\begin{figure*}
\centering
\input{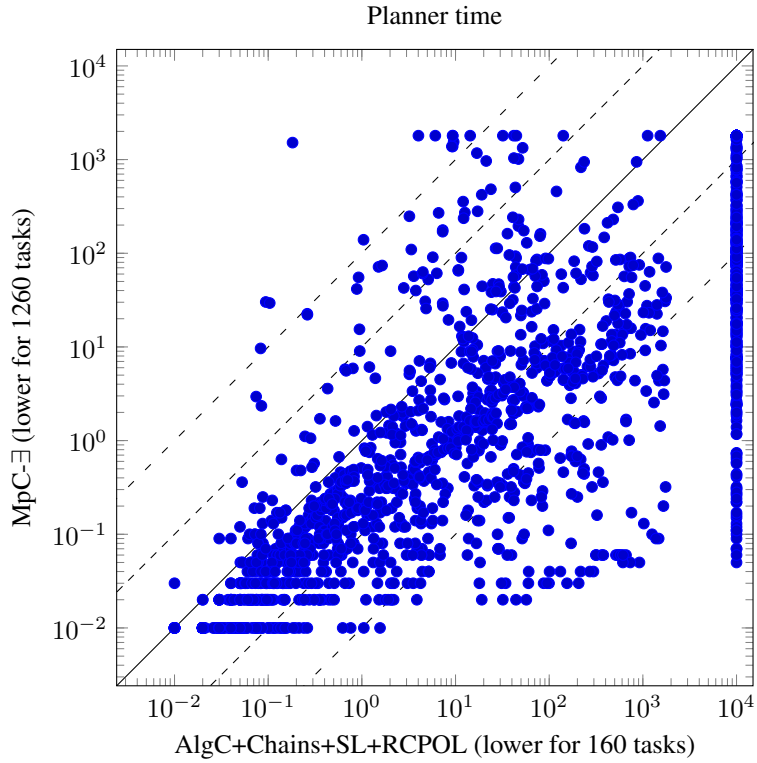}
\caption{Planner time comparison between MpC with $\exists$-step and our best encoding}
\end{figure*}
\begin{figure*}
\centering
\input{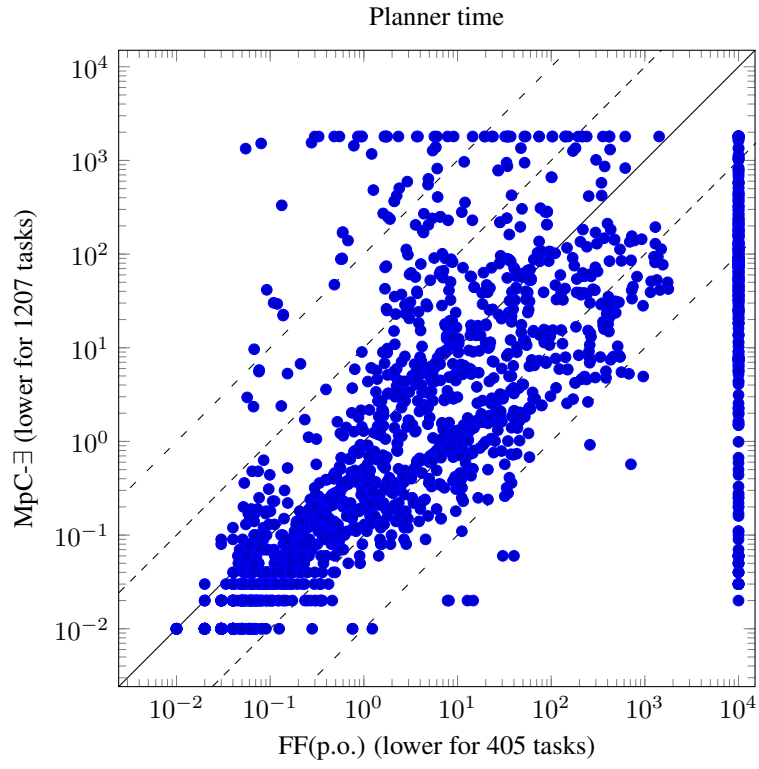}
\caption{Planner time comparison between MpC with $\exists$-step and FF(p.o.)}
\end{figure*}

\begin{figure*}
\centering
\input{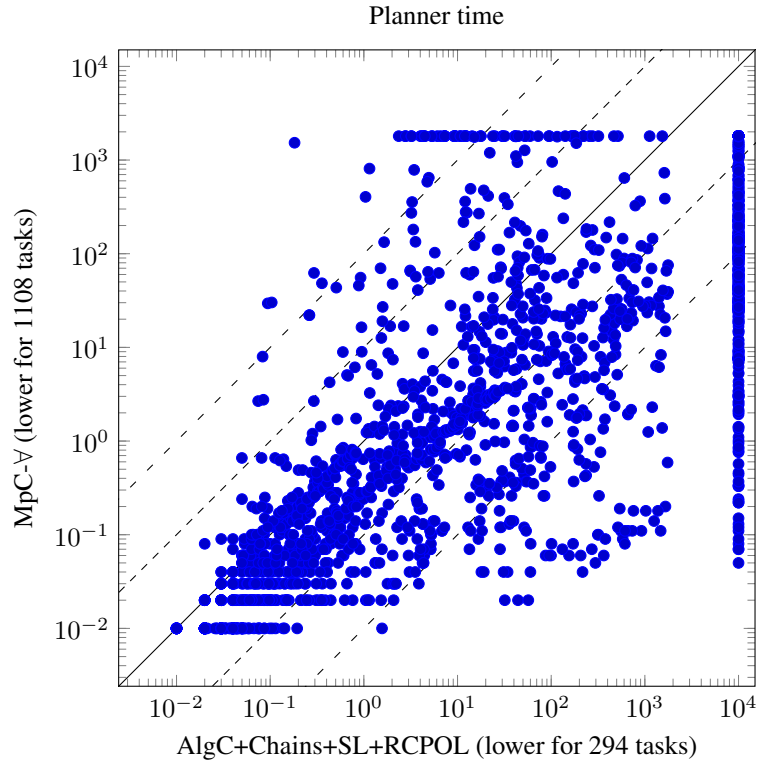}
\caption{Planner time comparison between MpC with $\forall$-step and our best encoding}
\end{figure*}

\begin{figure*}
\centering
\input{appendix/planner_time-ff-pref-shr-MpC-A-C}
\caption{Planner time comparison between MpC with $\forall$-step and FF(p.o.)}
\end{figure*}

\end{document}